%% file: root.tex
\documentclass[12pt]{elsarticle}
\journal{Autonomous Robots and Systems}
\usepackage{amsmath} %math environment like align
\usepackage[title]{appendix}
\usepackage{xpatch}
\usepackage[table]{xcolor}
\newcounter{secctr}
\xpretocmd{\section}{\refstepcounter{secctr}}{}{}
\usepackage{tikz}
\usepackage{hyperref}
\usetikzlibrary{calc}
\usepackage{tabularx}
\usepackage{multirow}
\usepackage[ruled,vlined,boxed,commentsnumbered,linesnumbered]{algorithm2e}
\usepackage{xspace}
\usepackage{amsthm,amssymb}

\input{util/util.tex}

\input{util/util-tikz.tex}

\input{util/theorem_proof_journal.tex}

\def\rrt{\texttt{RRT}\xspace}
\def\pdst{\texttt{PDST}\xspace}
\def\kpiece{\texttt{KPIECE}\xspace}
\def\sst{\texttt{SST}\xspace}
\def\irreducible{\texttt{[Irreducible]}\xspace}

\def\RobotLL{\Robot_{L}^N}

\input{src/definitions.tex}

\begin{document}
\begin{frontmatter}
%\title{Dimensionality Reduction for Motion Planning using Irreducible Path Spaces}

\title{\LARGE Motion Planning in Irreducible Path Spaces}
\author[aist]{Andreas Orthey\corref{cor1}}
\ead{andreas.orthey@aist.go.jp}
\author[laas]{Olivier Roussel}
\ead{olivier.roussel@laas.fr}
\author[laas]{Olivier Stasse}
\ead{ostasse@laas.fr}
\author[laas]{Michel Ta\"ix}
\ead{michel.taix@laas.fr}

\cortext[cor1]{Corresponding author}
\address[aist]{CNRS-AIST JRL (Joint Robotics Laboratory)\\
National Institute of Advanced Industrial Science and Technology (AIST)\\
Tsukuba Central 2, 1-1-1 Umezono, Ibaraki 305-8568 Japan\\
}
\address[laas]{ CNRS, LAAS,7 av. du Colonel Roche, F-31400,
Toulouse, France, Univ de Toulouse, LAAS, F-31400, Toulouse,
France}

%\thanks{This research has received funding from the European Union Seventh
%Framework Programme (FP7/2007 - 2013) under grant agreement n\textordmasculine\ 611909,
%KoroiBot.}}
\begin{abstract}

  The motion of a mechanical system can be defined as a path through its
  configuration space. Computing such a path has a computational complexity
  scaling exponentially with the dimensionality of the configuration space. We propose to reduce the dimensionality of the configuration
  space by introducing the irreducible path --- a path having a minimal swept
  volume. The paper consists of three parts: In part I, we define the space of
  all irreducible paths and show that planning a path in the irreducible path
  space preserves completeness of any motion planning algorithm. In part II, we
  construct an approximation to the irreducible path space of a serial kinematic
  chain under certain assumptions.  In part III, we conduct motion planning
  using the irreducible path space for a mechanical snake in a turbine
  environment, for a mechanical octopus with eight arms in a pipe system and for
  the sideways motion of a humanoid robot moving through a room with doors and
  through a hole in a wall. We demonstrate that the concept of an irreducible
  path can be applied to any motion planning algorithm taking curvature
  constraints into account. 

\end{abstract}
%%%%%%%%%%%%%%%%%%%%%%%%%%%%%%%%%%%%%%%%%%%%%%%%%%%%%%%%%%%%%%%%%%%%%%%%%%%%%%%%


\end{frontmatter}

\input{src/definitions.tex}
\input{src/introduction.tex}

\input{src/related-work.tex}

%%%%%%%%%%%%%%%%%%%%%%%%%%%%%%%%%%%%%%%%%%%%%%%%%%%%%%%%%%%%%%%%%%%%%%%%%%%%%%%%%
%PART I
\input{src/irreducible-paths.tex}

%%%%%%%%%%%%%%%%%%%%%%%%%%%%%%%%%%%%%%%%%%%%%%%%%%%%%%%%%%%%%%%%%%%%%%%%%%%%%%%%%
%PART II
\input{src/irreducible-motion-planning.tex}
\input{src/irreducible-algorithm.tex}
%%%%%%%%%%%%%%%%%%%%%%%%%%%%%%%%%%%%%%%%%%%%%%%%%%%%%%%%%%%%%%%%%%%%%%%%%%%%%%%%%
%PART III
\input{src/experiments.tex}
\input{src/experiments-results.tex}
\input{src/conclusion.tex}

\bibliographystyle{unsrtnat}
\bibliography{bib/general,bib/mathematics}

\newpage

\begin{appendices}
\input{src/irreducible-proof-single.tex}
\input{src/irreducible-proof-multi.tex}

\end{appendices}
\end{document}

%% file: util/util.tex
\usepackage{siunitx}

\usepackage{xspace}
\usepackage{mathtools}
\usetikzlibrary{arrows, decorations.markings}
\def\z{ \ensuremath{z} }

\newcommand\W{ \ensuremath{\mathcal{W}} }
\newcommand\SE{ \ensuremath{SE}\xspace}

\DeclareMathOperator{\atan}{atan}

\newcommand\dtau{\ensuremath{\dot{\tau}}\xspace}

\def\Fqq{\ensuremath{\mathcal{F}_{q_I,q_G}}\xspace}
\def\Ips{\ensuremath{\mathcal{\hat{I}}}\xspace}

\newcommand\env{\ensuremath{\mathbf{E}}}

\newcommand\I{\ensuremath{\mathcal{I}}}
\newcommand\SV{\ensuremath{V}}

\DeclareMathOperator{\asin}{asin}
\DeclareMathOperator{\acos}{acos}

\def\x{ \ensuremath{x} }

\def\C{ \ensuremath{\mathcal{C}} }

\def\R{\ensuremath{\mathbb{R}}\xspace}

\def\Robot{\ensuremath{\mathcal{R}}\xspace}

\def\Dt{\ensuremath{\mathcal{\tilde{D}}}\xspace}

\def\F{\ensuremath{\mathcal{F}}\xspace}

\def\F{\ensuremath{\mathcal{F}}\xspace}

\def\K{\ensuremath{\mathcal{K}}\xspace}

\newcounter{todoCtr}

\usepackage{xspace}
\usetikzlibrary{patterns}
\xdefinecolor{darkgreen}{rgb}{0,0.60,0}
\xdefinecolor{lightgray}{rgb}{0.9,0.9,0.9}
\tikzset{
    -|/.style={
        to path={
            (perpendicular cs: horizontal line through={(\tikztostart)},
                                 vertical line through={(\tikztotarget)})
            -- (\tikztotarget) \tikztonodes
        }
    }
}
\tikzstyle{taskRect}=[draw, color=black!70, fill=black!7, rectangle, rounded
corners, thick, minimum width=3cm, minimum height=0.6cm]
\tikzstyle{borderRect}=[draw, color=red!70, fill=gray!7, rectangle, rounded
corners, thick, minimum width=8cm, minimum height=3cm]
\tikzstyle{testBRect}=[draw, color=darkgreen!100, fill=gray!7, rectangle, rounded
corners, thick, minimum width=8cm, minimum height=3cm]
\tikzstyle{stdBR}=[draw, color=black, fill=gray!7, rectangle, rounded
corners, thick, minimum width=8cm, minimum height=3cm]
\tikzstyle{boundingBox}=[draw, color=red!100, fill=gray!7, rectangle, rounded
corners, thick, minimum width=3cm, minimum height=1cm]

\tikzstyle{arrow}=[->, draw, thick]
\tikzstyle{darrow}=[->, draw, ultra thick]
\tikzstyle{uarrow}=[-, draw, ultra thick]
\tikzstyle{linkArrow} = [thick, decoration={markings,mark=at position
   1 with {\arrow[semithick]{open triangle 60}}},
   double distance=0.3cm, shorten >= 0.0pt,
   preaction = {decorate}]

\tikzstyle{hlBox}=[draw, color=red!70, fill=black!7, rectangle, rounded
corners, ultra thick, minimum width=3cm, minimum height=0.6cm]
\tikzstyle{nUnit}=[draw, color=black!70, fill=black!7, circle, rounded
corners, ultra thick, minimum width=1cm, minimum height=0.2cm]
\tikzstyle{subspace}=[draw, circle, color=black, fill=black!7, thick, pattern=north east lines, pattern color=black!20, minimum width=1cm, minimum height=0.2cm]

\tikzstyle{endeffectorArrow} = [thick,
                decoration={markings,mark=at position 1 with
                {\arrow[scale=1.4,thick,black]{[}}},
                    postaction={decorate},
double distance=3pt, shorten >= 2.5pt]
\newlength{\hatchspread}
\newlength{\hatchthickness}
\newlength{\hatchshift}
\newcommand{\hatchcolor}{}
\tikzset{hatchspread/.code={\setlength{\hatchspread}{#1}},
         hatchthickness/.code={\setlength{\hatchthickness}{#1}},
         hatchshift/.code={\setlength{\hatchshift}{#1}},% must be >= 0
         hatchcolor/.code={\renewcommand{\hatchcolor}{#1}}}
\tikzset{hatchspread=3pt,
         hatchthickness=0.4pt,
         hatchshift=0pt,% must be >= 0
         hatchcolor=black}
\pgfdeclarepatternformonly[\hatchspread,\hatchthickness,\hatchshift,\hatchcolor]% variables
   {custom north west lines}% name
   {\pgfqpoint{\dimexpr-2\hatchthickness}{\dimexpr-2\hatchthickness}}% lower left corner
   {\pgfqpoint{\dimexpr\hatchspread+2\hatchthickness}{\dimexpr\hatchspread+2\hatchthickness}}% upper right corner
   {\pgfqpoint{\dimexpr\hatchspread}{\dimexpr\hatchspread}}% tile size
   {% shape description
    \pgfsetlinewidth{\hatchthickness}
    \pgfpathmoveto{\pgfqpoint{0pt}{\dimexpr\hatchspread+\hatchshift}}
    \pgfpathlineto{\pgfqpoint{\dimexpr\hatchspread+0.15pt+\hatchshift}{-0.15pt}}
    \ifdim \hatchshift > 0pt
      \pgfpathmoveto{\pgfqpoint{0pt}{\hatchshift}}
      \pgfpathlineto{\pgfqpoint{\dimexpr0.15pt+\hatchshift}{-0.15pt}}
    \fi
    \pgfsetstrokecolor{\hatchcolor}
    \pgfusepath{stroke}
   }

\tikzset{
 ->-/.style={
  decoration={
   markings,
   mark=at position 0.5 with {\arrow{>}{angle 90};}
  },
  postaction={decorate}
 }
}

\tikzset{
 ->>-/.style={
  decoration={
   markings,
   mark=at position 0.55 with {\arrow{>>}{angle 90};}
  },
  postaction={decorate}
 }
}
\usepackage{wrapfig}

%% file: util/util-tikz.tex
\usepackage{pgfkeys}
\pgfkeys{
/drawCapsule/.is family, /drawCapsule,
default/.style={line width = 1pt, color=black!30},
line width/.estore in = \linewidthC,
color/.estore in = \colorC,
}

\pgfdeclaredecoration{simple line}{initial}{
  \state{initial}[width=\pgfdecoratedpathlength-1sp]{\pgfmoveto{\pgfpointorigin}}
  \state{final}{\pgflineto{\pgfpointorigin}}
}

\tikzset{
   shift left/.style={decorate,decoration={simple line,raise=#1}},
   shift right/.style={decorate,decoration={simple line,raise=-1*#1}},
   node left/.style 2 args={
        decorate,decoration={raise=#1,markings,mark = at
        position 0.5 with {\node[circle, color=black,fill = white] {#2};}}
        },
}

\def\LeftShiftLine[#1][#2][#3][#4]{
    \begin{scope}[>=latex] % redef arrow for dimension lines
        \draw[dashed,|-|] #1 edge[shift left=#3] #2;
        \draw[] #1 edge[node left={#3}{#4}] #2;

    \end{scope}
}

%% file: util/theorem_proof_journal.tex
\usepackage{amsfonts} %mathbb, etc
\usepackage{amsthm,amssymb}

\theoremstyle{definition}
\theoremstyle{plain}

\newtheorem{definition}{Definition}
\newtheorem{theorem}{Theorem}
\newtheorem{lemma}{Lemma}

\newtheorem{example}{Example}

%% file: src/definitions.tex
\def\rrt{\texttt{RRT}\xspace}
\def\pdst{\texttt{PDST}\xspace}
\def\kpiece{\texttt{KPIECE}\xspace}
\def\sst{\texttt{SST}\xspace}
\def\irreducible{\texttt{[Irreducible]}\xspace}

\def\RobotLL{\Robot_{L}^N}

\def\Pp{P_0}
\def\dP{Q_0}
\def\tl{\bar{\mathbf{\theta}}}
\def\Kcone{\K_{\tl}(p_0)}
\def\Fk{\F_{\kappa_0}}

\def\Pn{P_N}
\def\dPn{Q_N}
\def\Fkn{\F_{\kappa_N}}
\def\sumL{Nl_0}

\def\ti{\theta_{i}}
\def\tii{\theta_{i-1}}

\def\dkn{d_{\kappa_N}}
\def\stl{\sin^2{\tl}}
\def\sstl{\sin{\tl}}
\def\stheta{\theta}

\def\btheta{\boldsymbol{\theta}}
\def\bgamma{\boldsymbol{\gamma}}

\def\x{\mathbf{e}_x}
\def\y{\mathbf{e}_y}
\def\z{\mathbf{e}_z}
\def\xe{\x}
\def\ye{\y}
\def\ze{\z}

\def\ex{\mathbf{e}_1}
\def\ey{\mathbf{e}_2}
\def\ez{\mathbf{e}_3}
\def\tc{s_{last}}
\def\tn{s_{next}}
\def\Rab{\mathbf{R}}
\def\pb{p_I}
\def\pa{p_W}
\def\xl{x_L}
\def\pxy{p_{xy}}
\def\pzx{p_{zx}}
\def\pa{p_W}

\def\Dt{\Delta t}
\def\Ry{\mathbf{R}_Y}
\def\Rz{\mathbf{R}_Z}

%% file: src/introduction.tex
\def\AF{A_{\F}}
\def\AI{A_{\I}}
\section{Introduction\label{sec:introduction}}

Motion planning \citep{lavalle_2006} has been succesfully applied to many
mechanical systems with applications in computer graphics, humanoid robotics or
protein folding. The key idea of motion planning is to define the motion of a
mechanical system as a path through its configuration space. Given two
configurations, the goal of motion planning is to construct a motion planning
algorithm computing a path connecting the two configurations.

Real-world systems like mechanical snakes or humanoid robots have many degrees
of freedom (DoF) and therefore a high-dimensional configuration space. The
higher the dimensionality of the configuration space, the more time a motion
planning algorithm needs to find a solution. In fact, any motion planning
algorithm has a computational complexity scaling exponentially with the
dimensionality of the configuration space \citep{reif_1979}.

A key challenge in motion planning is therefore to reduce the dimensionality of
the configuration space. Dimensionality reduction of configuration spaces has
been addressed by several researchers
\citep{dalibard_2011b,allen_2014,zhang_2009}, but the results only
apply in special cases. In fact, there is no general approach to reduce the
dimensionality of a configuration space in a principled way. 

Our work contributes to this effort by introducing the irreducible path
\citep{orthey_2015a}, a configuration space path having a minimal swept volume
\footnote{The swept volume is the volume occupied by the body of a mechanical system while
moving along a path \citep{himmelstein_2010}}. The space of all those minimal
swept volume paths creates the irreducible path space. 
%In
%Theorem \ref{fundamental-irreducible-theorem} we show that if there exists a
%feasible path, then there exists an irreducible feasible path. 
Our main result is Theorem \ref{thm:complete}, which shows that replacing the
full space of continuous paths with the space of irreducible paths preserves
completeness of any motion planning algorithm. This is advantageous because
computing an irreducible path can often be done in a lower dimensional
configuration space, thereby reducing the computational complexity.

The paper consists of three parts. In Part I we define the irreducible path and
the irreducible path space. We discuss the partitioning of the irreducible path
space under equivalent swept volumes.  We then prove the completeness of any
motion planning algorithm using the irreducible path space in Theorem
\ref{thm:complete}.  We note that those concepts apply to any functional space:
the space of all dynamical feasible paths, all statically stable paths, all
torque constraint paths or all collision-free paths. For sake of simplicity, we
focus here exclusively on collision-free paths.

In Part II, we approximate the irreducible path space of a serial kinematic chain.
The main idea is the following: if the root link moves on a curvature-constraint
path, then all the sublinks can be projected into the swept volume of the
root link. Therefore, we can ignore the sublinks and we can thereby reduce the
dimensionality of the configuration space. This has been visualized in the case
of a serial kinematic chain in the plane in Fig.  \ref{fig:irrconcept}.

In Part III, we apply the reduction of the serial kinematic chain to four
different mechanical systems: an idealized serial kinematic chain on
$\SE(2)\times \R^3$, $\SE(2)\times \R^6$ and $\SE(3)\times \R^{12}$, a
mechanical snake in a turbine environment on $\SE(3) \times \R^{16}$, a
mechanical octopus in a tunnel system on $\SE(3) \times \R^{80}$ and a humanoid
robot moving sideways on $\SE(2) \times \R^{19}$ through a room with doors and
through a hole in a wall.

This work extends previous results in \cite{orthey_2015a}, where we introduced
the irreducible path and applied it to the sideway motion of a humanoid
robot. Section \ref{sec:irrpath} is based on \citep{orthey_2015a}, it has been
revised and the proofs have been simplified. 

\begin{figure*}
  \centering
  \def\ww{0.24} \def\hh{0.25}
  \includegraphics[width=\ww\linewidth,height=\hh\linewidth]
  {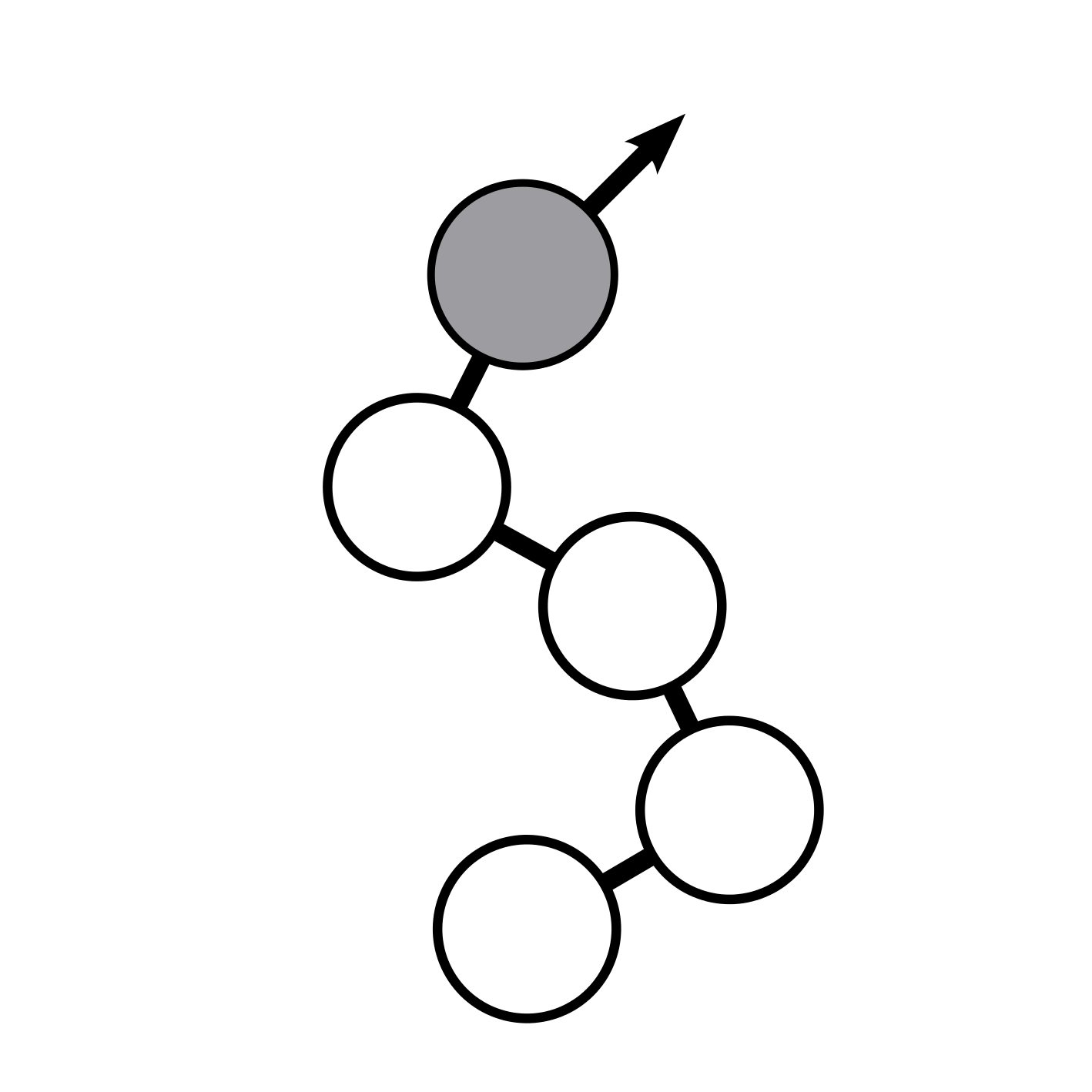}
  \includegraphics[width=\ww\linewidth,height=\hh\linewidth]
  {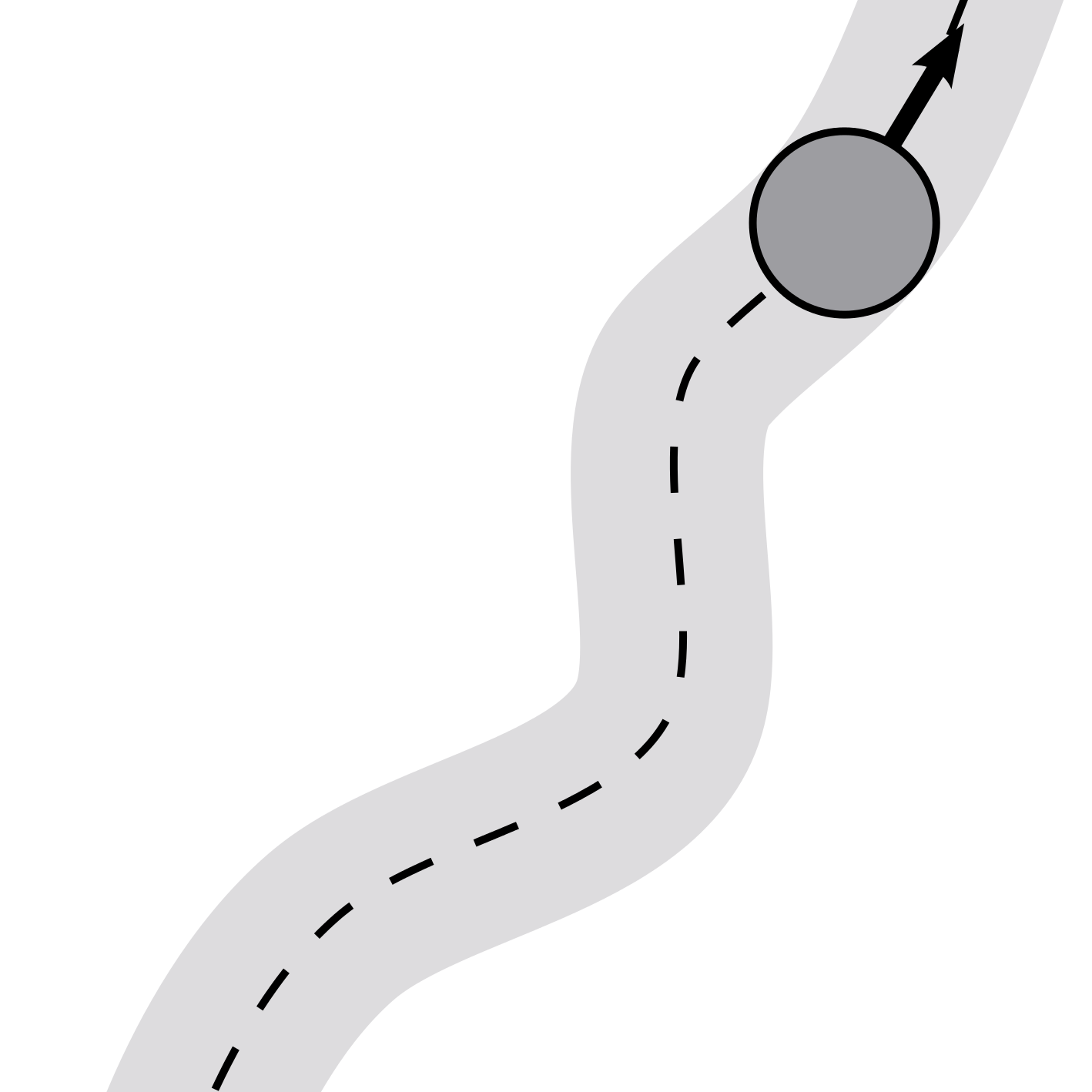}
  \includegraphics[width=\ww\linewidth,height=\hh\linewidth]
  {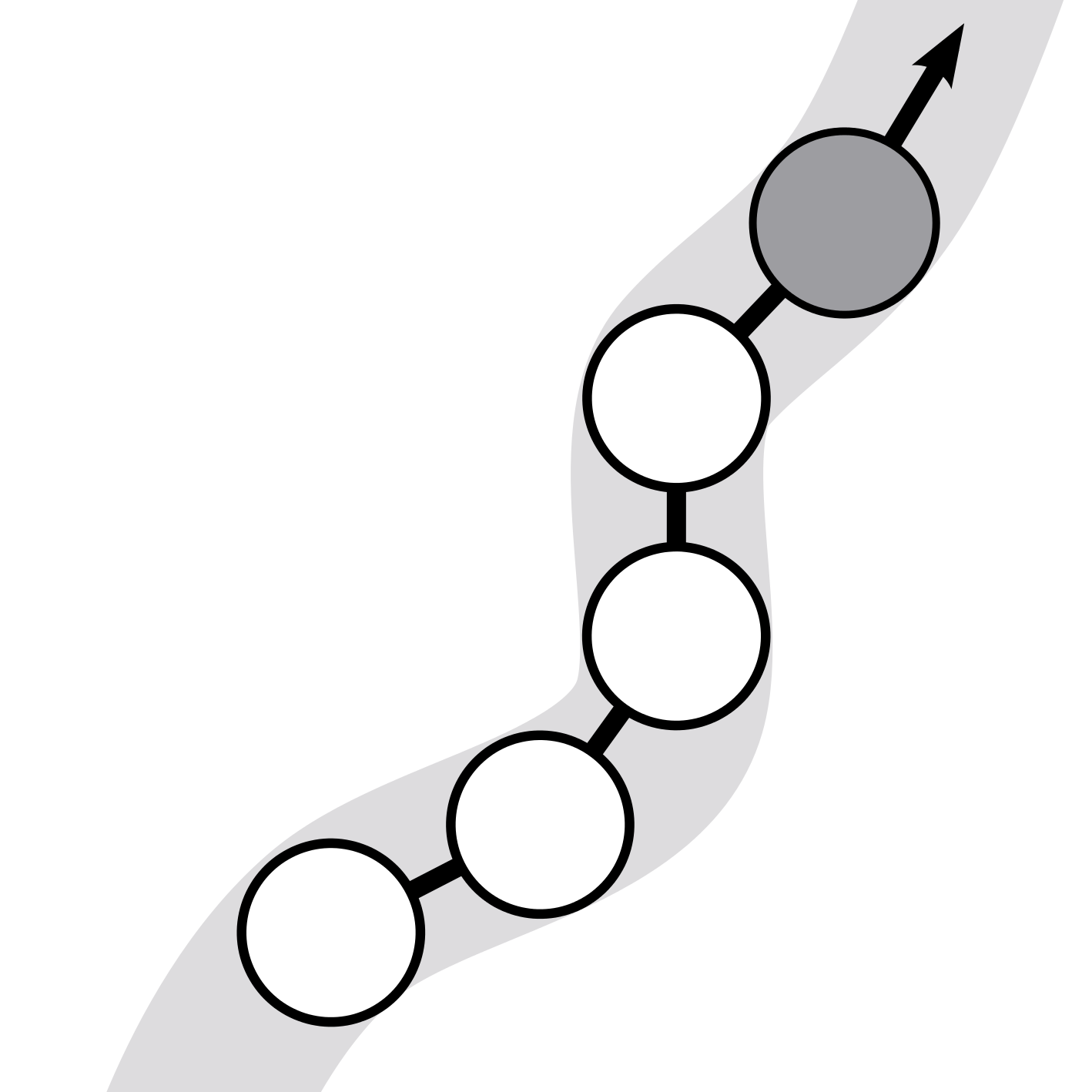}
  \includegraphics[width=\ww\linewidth,height=\hh\linewidth]
  {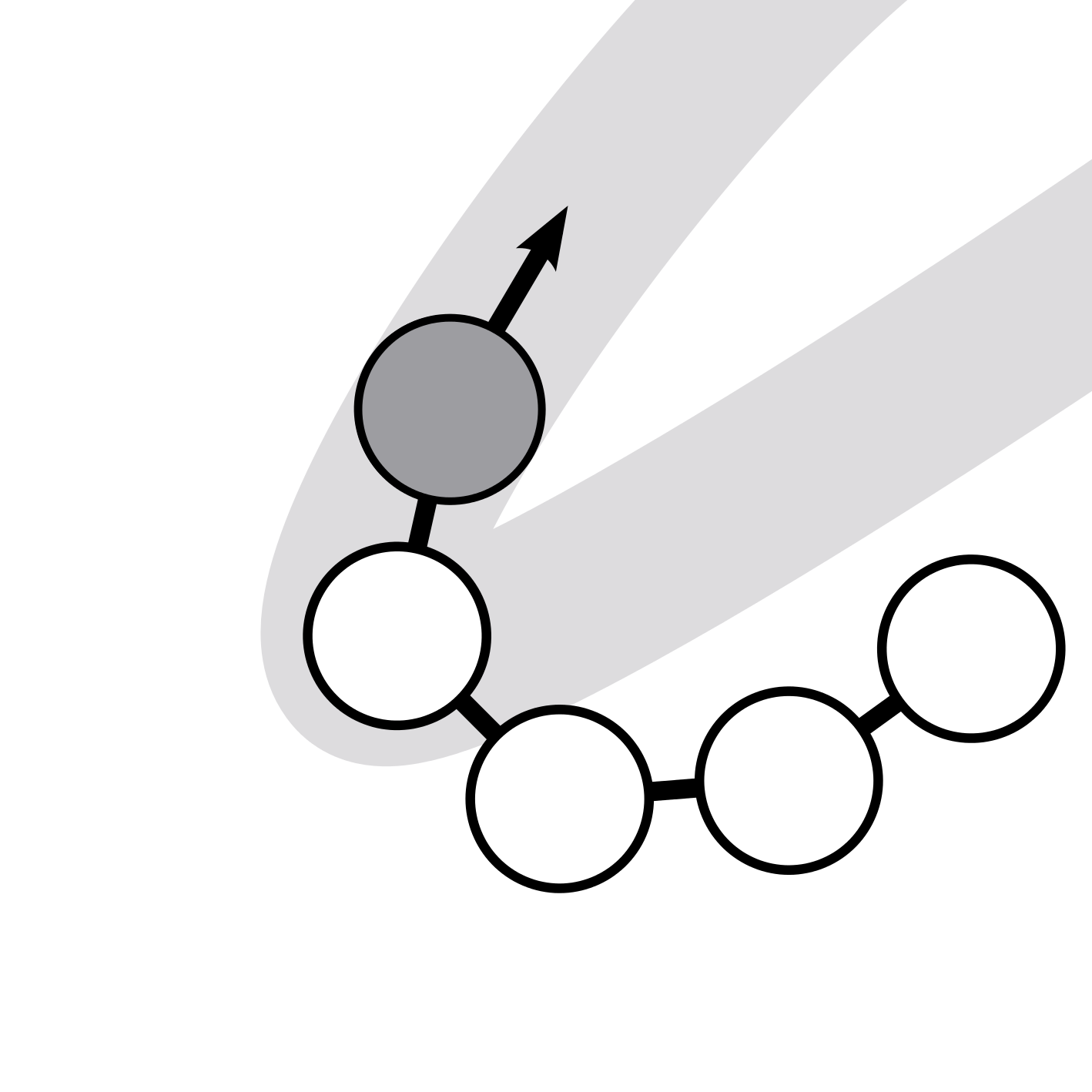}

  \caption{{\bf Left:} A serial kinematic chain with a root link (grey) and four
    sublinks (white). {\bf Middle left:} Root link moves on a curvature
    constraint path and sweeps a volume (lightgrey). {\bf Middle Right:} Our algorithm projects the sublinks
    into the swept volume of the root link. {\bf Right:} If the curvature of the
    root link path is too high, the algorithm will fail to project the sublinks.
\label{fig:irrconcept}}

\end{figure*}

%% file: src/related-work.tex
\section{Related Work\label{sec:relatedwork}}

Dimensionality reduction of configuration spaces has been extensively studied in
the motion planning literature. \cite{dalibard_2011b} have used a principal
component analysis (PCA) to locally reduce dimensions of small volume and
thereby bias random sampling. In the context of manipulation planning,
\cite{ciocarlie_2007} and \cite{allen_2014} have introduced the eigengrasp to
identify a low-dimensional representation of grasping movements.
\cite{mahoney_2010} perform a PCA for a high-dimensional cable robot by sampling
deformations. The idea being that many configurations occupy similar volumes in
workspace. \cite{kabul_2007} plan the motion of a cable by first planning a
motion for the head. \cite{salzman_2015} approximate the manifold of
self-collision free configurations of the robot, thereby projecting the planning
problem onto a lower dimensional sub-manifold in the configuration space. 

The approach closest to ours is a subspace decomposition scheme by
\cite{zhang_2009}: an initial path is planned by using only one large subpart of
the robot. This initial path is then deformed to account for the remaining
links. However, there is no justification or guarantee for being complete. 

%Paper "Motion Planning of Human-Like Robots using Constrained Coordination": The
%approach use a decomposition in sub-problem, after computing a initial solution
%for a subpart of the robot, a second step (refine technique) is used to modified
%the initial solution, by consequence the initial solution for the subpart can be
%(after modification) in collision and you don't have garantiee of convergence
%even if a solution exist.

\cite{bereg_2005} seem to be the first to introduce the term
reducibility of motions. They consider sweeping of disks along a planar curve,
whereby the volume of a disk swept along a path is reduced if it is a subset of
the swept volume of the same disk swept along another path. We generalize this
concept to arbitrary configuration spaces.

In Sec. \ref{mpskc} we establish that sublinks of a serial kinematic chain can
be projected into the swept volume if the root link moves on a curvature
constrained path. Curvature constrained paths are one of the central objects of
study in differential geometry \citep{banchoff_2015}. Our work builds upon
work by \cite{ahn_2012} who compute the reachable regions for
curvature-constraint motions inside convex polygons. A generalization of these
ideas to 3D has been investigated by \cite{guha_2005} who discuss curvature and
torsion constraints on space curves in the context of data point approximation. 

Our applications consider motion planning for a mechanical snake, a mechanical
octopus and a humanoid robot.

The mechanism and locomotion system for snake robots have been studied by
\cite{HiroseY09}. Path planning for snake robots has been investigated in
relatively few papers, some of whom are classical approaches using numerical
potential fields \citep{Conkur-2008}, genetic algorithms \citep{LIU-2004} or
Generalized Voronoi Graphs \citep{Choset-1998}. The idea of dimensionality
reduction for snake robots has been studied by \cite{Rollinson-2011}, who define
a frame consistent with the overall shape of the robot in all configurations.
\cite{Cappo-2014} plan a path only for a portion of the snake robot. 

Octopus robots have been built by \cite{sfakiotakis_2014} and
\cite{cianchetti_2015}, and its locomotion behavior has been intensively
investigated by \cite{calisti_2015}. However, there has been no demonstration of
motion planning for an octopus robot. We concentrate here on motion planning
using jet propulsion in narrow environments like a system of pipes. 

Motion planning for humanoid robots is a well studied field \citep{harada_2010a}.
Applications range from manipulation planning in kitchen environments
\citep{vahrenkamp_2009}, contact planning in constrained environments
\citep{escande_2013}\citep{hauser_2006}\citep{deits_2014} to ladder climbing tasks
\citep{zhang_2014}. Since general multi-contact planning has a high run-time,
researchers have tried to decompose the problem by first planning for simple
geometrical shapes. A common approach is first to plan for a sliding box on a
floor, then generate footsteps along the box path
\citep{dalibard_2011b}\citep{elkhoury_2013}. Such an approach does not work in the
environments we consider, and our approach can be seen as a generalization of
the decomposition to include the original geometry of the robot.

%% file: src/irreducible-paths.tex
\input{images/irreducibility-toy-example.tex}
%%%%%%%%%%%%%%%%%%%%%%%%%%%%%%%%%%%%%%%%%%%%%%%%%%%%%%%%%%%%%%%%%%%%%%%%%%%%%%%
\section{Irreducible Paths\label{sec:irrpath}}
%%%%%%%%%%%%%%%%%%%%%%%%%%%%%%%%%%%%%%%%%%%%%%%%%%%%%%%%%%%%%%%%%%%%%%%%%%%%%%%

The irreducible path is a path of minimal swept volume \cite{orthey_2015a}. In
this section we define the irreducible path space, we discuss why the
irreducible path space is important (Sec. \ref{sec:irrfeasible}), and we
investigate the internal structure in Sec. \ref{sec:irrstructure}. Then we prove
completeness (Sec.  \ref{sec:irrcompleteness}: If a motion planning algorithm is
complete using all paths, then it is complete using only irreducible paths.
Finally, we discuss generalizations to dynamically feasible paths in Sec.
\ref{sec:irrdynamics}.

Let $\Robot$ be a robot and $\C$ its configuration space, the space of all transformations applicable to
$\Robot$ \citep{lavalle_2006}. Let further $q_I \in \C$ be the initial
configuration and $q_G \in \C$ be the goal
configuration.

A motion planning algorithm needs to compute a continuous path between $q_I$ and
$q_G$. The solution space is therefore defined in terms of functional
spaces \citep{farber_2003,hatton_2011}. We first define the space of all paths in $\C$
as

\begin{equation}
        \begin{aligned}
          \Phi = \{ \phi: [0,1] \rightarrow \C |\ \phi\ \text{continuous}\}
        \end{aligned}
\end{equation}

\noindent We denote the path space between $q_I$ and $q_G$ as

\begin{equation}
        \begin{aligned}
          \Fqq = \{ \tau \in \Phi\ |\ \tau(0) = q_I, \tau(1) = q_G \} \label{eq:pathspace}
        \end{aligned}
\end{equation}

\noindent For the purpose of this paper we will abbreviate $\F = \Fqq$ assuming that some
$q_I,q_G \in \C$ have been choosen. $\F$ will be called the (full) path space.

If the robot $\Robot$ follows a path $\tau \in \F$, the body of the robot will
sweep a volume. We will denote this swept volume by $\SV(\tau)$. We then define
an irreducible path as

\begin{definition}[Irreducible Path]

A path $\tau' \in \F$ is called reducible by $\tau$, if there exist $\tau \in
\F$ such that $\SV(\tau) \subset \SV(\tau')$. Otherwise $\tau'$ is called
\emph{irreducible}. \label{def:irrpath}

\end{definition}

All irreducible paths define the irreducible path space.

\begin{definition}[Irreducible Path Space]

The space of all irreducible configuration space paths is

\begin{equation}
  \begin{aligned}
    \Ips = \{\tau \in \F\ |\ \tau \text{ is irreducible}\}
  \end{aligned}
\end{equation}
\end{definition}

\begin{example}{[2-dof robot]}

In Fig. \ref{irreducible-explanation} we consider a 2-link 2-dof robot, which
can move along a straight line and which can rotate its second link around a
pivot point. In the second column, the robot is shown in the workspace, once for
its initial configuration $q_I$, once for its goal configuration $q_G$. In the
first column, we show three configuration space paths $\tau_1,\tau_2,\tau_3$
connecting $q_I$ to $q_G$. On the right we show the corresponding swept volumes
in workspace for each path. Applying the definition of irreducibility, we have
that $\tau_2$ and $\tau_3$ are reducible by $\tau_1$, while $\tau_1$ is irreducible.

\end{example}

%%%%%%%%%%%%%%%%%%%%%%%%%%%%%%%%%%%%%%%%%%%%%%%%%%%%%%%%%%%%%%%%%%%%%%%%%%%%%%%
\subsection{Feasibility of Irreducible Path Space\label{sec:irrfeasible}}
%%%%%%%%%%%%%%%%%%%%%%%%%%%%%%%%%%%%%%%%%%%%%%%%%%%%%%%%%%%%%%%%%%%%%%%%%%%%%%%

The importance of the irreducible path space comes from the following claim:
If every irreducible path is infeasible, then all paths are infeasible. We prove
this claim in Theorem \ref{fundamental-irreducible-theorem}.

Let us denote by $\env$ the environment, the set of obstacle regions in $\R^3$
\citep{lavalle_2006}. We say that a path $\tau \in \F$ is called feasible in a
given environment $\env$, if $\SV(\tau) \cap \env = \emptyset$. 

\begin{theorem}
\label{thm:feasible}

Let $\tau \in \Ips,\tau' \in \F$ be such that $\SV(\tau) \subset \SV(\tau')$, i.e.
$\tau'$ is reducible by $\tau$.

\begin{center}
        (1) If $\tau$ is infeasible $\Rightarrow$ $\tau'$ is infeasible\\
        (2) If $\tau'$ is feasible $\Rightarrow$ $\tau$ is feasible
\end{center}

\end{theorem}

\begin{proof}
        Let $S=\SV(\tau)$ and $S'=\SV(\tau')$. 
        %$S$ is feasible if $S \cap \env =
        %\emptyset$, whereby $\env$ is the environment. \newline
        (1) Let $S \cap \env \neq \emptyset$, then there exists $v \in S \cap \env$. 
        Since $S \subset S'$, $v$ has to be
        in $S'$. But $v$ is also in $\env$, such that $S' \cap \env$ has to
        contain at least $v$ and is therefore not empty.\newline
        (2) Let $S' \cap \env = \emptyset$. Since $S
        \subset S'$, it follows that $S \cap \env = \emptyset$, which shows that $\tau$ is feasible.
\end{proof}

We will say that a path space $\F$ is feasible, if there exists at least one
feasible $\tau \in \F$. If there is no feasible $\tau \in \F$, then we say that
the space itself is \emph{infeasible}. By Theorem \ref{thm:feasible} it follows that

\begin{theorem}
If $\Ips$ is infeasible, then $\F$ is infeasible.
\label{fundamental-irreducible-theorem}
\end{theorem}

\begin{proof}

        Let $\tau \in \F$. There are two cases: either (1) there exists a $\tau'
        \in \Ips$ such that $\SV(\tau')\subset \SV(\tau)$. Then $\tau$ is
        infeasible by Theorem \ref{thm:feasible}. Or (2) there is no $\tau' \in
        \Ips$ such that $\SV(\tau') \subset \SV(\tau)$. Then $\tau$ is by
        definition in $\Ips$ and therefore infeasible.

\end{proof}

%%%%%%%%%%%%%%%%%%%%%%%%%%%%%%%%%%%%%%%%%%%%%%%%%%%%%%%%%%%%%%%%%%%%%%%%%%%%%%%
\subsection{Structure of Irreducible Path Space\label{sec:irrstructure}}
%%%%%%%%%%%%%%%%%%%%%%%%%%%%%%%%%%%%%%%%%%%%%%%%%%%%%%%%%%%%%%%%%%%%%%%%%%%%%%%

The irreducible path space $\Ips$ can be partitioned
into equivalence classes of paths with equivalent swept
volumes.

\begin{definition}[Swept Volume Equivalence]
  Two irreducible paths $\tau,\tau' \in \Ips$ are swept-volume equivalent $\tau
  \simeq \tau'$
  if $\SV(\tau) = \SV(\tau')$
\end{definition}

This equivalence relation gives rise to equivalence classes of swept-volume
equivalent trajectories. This includes all injective and continuous
time-reparameterizations $s: [0,T] \rightarrow [0,1]$ of a path.

Those equivalence classes partition the irreducible path space into a union of
disjoint path spaces. This is made precise by taking the quotient space
\citep{munkres_2000} of $\Ips$ as

\begin{equation}
  \begin{aligned}
    \I = \Ips /{\simeq}\label{eq:quotientspace}
  \end{aligned}
\end{equation}

i.e. all swept-volume equivalent paths in $\Ips$ are assigned to exactly one
path in the quotient space $\I$.

%%%%%%%%%%%%%%%%%%%%%%%%%%%%%%%%%%%%%%%%%%%%%%%%%%%%%%%%%%%%%%%%%%%%%%%%%%%%%%%
\subsection{Completeness of Irreducible Path Space\label{sec:irrcompleteness}}
%%%%%%%%%%%%%%%%%%%%%%%%%%%%%%%%%%%%%%%%%%%%%%%%%%%%%%%%%%%%%%%%%%%%%%%%%%%%%%%

We say that a motion planning algorithm is complete in $\F$ if it finds
a feasible path in $\F$ if one exists or correctly reports that none
exists. Imagine replacing $\F$ by $\I$. We claim that if a motion planning
algorithm is complete in $\F$ then it is complete in $\I$ and vice versa. 

\begin{theorem}

  A motion planning algorithm is complete in $\I$ iff it is complete in $\F$
  \label{thm:complete}

\end{theorem}

\begin{proof}

  %By definition, a motion planning algorithm is complete in $\F$, if it can find
  %a feasible path in $\F$ if one exists, or correctly reports that no such path
  %exists. Equivalently, a motion planning algorithm is complete if it can decide
  %if the path space $\F$ is feasible or infeasible. This is equivalent to
  %deciding if the irreducible path space $\I$ is feasible or infeasible: 

  To prove equivalence, we need to prove four statements. 
  \begin{enumerate}[(1)]
    \item $\I$ infeasible $\Rightarrow$ $\F$ infeasible,
    \item $\I$ feasible $\Rightarrow$ $\F$ feasible,
    \item $\F$ infeasible $\Rightarrow$ $\I$ infeasible and 
    \item $\F$ feasible $\Rightarrow$ $\I$ feasible.  
  \end{enumerate}
  
  (1) is true by Theorem
  \ref{fundamental-irreducible-theorem}. Statements (2) and (3) are true by inclusion and
  (4) is true by contraposition of Theorem
  \ref{fundamental-irreducible-theorem}.
        
\end{proof}

In light of Theorem \ref{thm:complete} we call the reduction from $\F$ to $\I$
a completeness-preserving reduction.

\begin{example}{[Completeness-preserving reduction]}

Let us demonstrate the completeness property by the example from Fig.
\ref{irreducible-explanation}. Imagine that $\F$ contains only
$\tau_1,\tau_2,\tau_3$ and $\I$ contains $\tau_1$. Imagine further that
$\tau_1$ is infeasible because it is in collision with some imagined obstacle.
Then $\I$ is infeasible. But since $\tau_2$ and $\tau_3$ are supersets of
$\tau_1$, they are infeasible too, so we see that $\F$ must be infeasible.
Imagine that $\tau_1$ is feasible.  Then $\I$ is feasible. Since $\I$ is
contained in $\F$, $\F$ must be feasible, too. 

\end{example}

%%%%%%%%%%%%%%%%%%%%%%%%%%%%%%%%%%%%%%%%%%%%%%%%%%%%%%%%%%%%%%%%%%%%%%%%%%%%%%%
\subsection{Generalization of Irreducible Path to Dynamics\label{sec:irrdynamics}}
%%%%%%%%%%%%%%%%%%%%%%%%%%%%%%%%%%%%%%%%%%%%%%%%%%%%%%%%%%%%%%%%%%%%%%%%%%%%%%%

If we consider the dynamics of the robot, we could obtain a situation where the
only dynamically feasible path is not irreducible. This could happen if
the momentum of the sublinks is crucial to solve a task. While in this paper we
focus on collision-free paths where the dynamics are not considered, the concept of an
irreducible path can in principle be generalized to incorporate dynamically feasible paths.

A path is dynamically feasible if it is a solution to the equation of motion of
the robot. We define $\F_D$ as the subspace of all paths being dynamically
feasible. Combined with the concept of an irreducible path we obtain

\begin{definition}[Dynamically Irreducible Path]

A path $\tau' \in \F_D$ is called dynamically reducible by $\tau$, if there exist $\tau \in
\F_D$ such that $\SV(\tau) \subset \SV(\tau')$. Otherwise $\tau'$ is called
\emph{dynamically irreducible}. \label{def:dynirrpath}

\end{definition}

We note that the previous discussion of completeness analogously 
applies to dynamically irreducible paths. In this paper, however, we consider
only the non-dynamical case. The characterization of the dynamically irreducible
path space is left for future work.

%The characterization of the
%dynamically irreducible path space and its application is left for future work. 

%% file: images/irreducibility-toy-example.tex
\begin{figure}

\newlength{\picwidth}
\setlength{\picwidth}{0.8\linewidth}
\centering

$\underbrace{
\includegraphics[width=0.3\picwidth,height=0.6\picwidth]{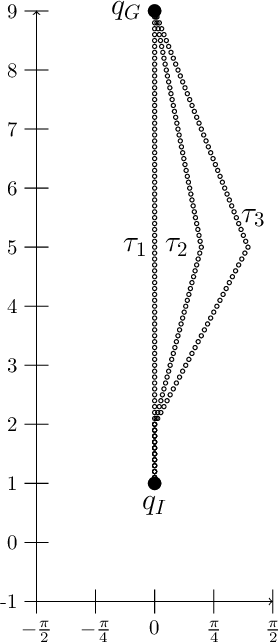}
}_{\C=\R \times [-\frac{\pi}{2},\frac{\pi}{2}]}
\underbrace{
\underbrace{
\includegraphics[width=0.15\picwidth]{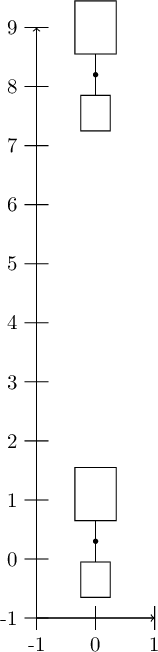}
}_{\SV(q_I),\SV(q_G)}
\underbrace{
\includegraphics[width=0.15\picwidth]{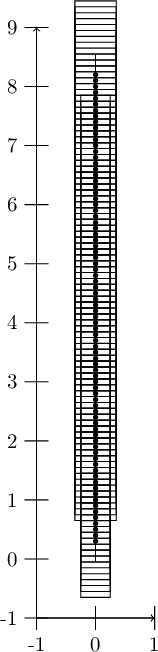}
}_{\SV(\tau_1)}
\underbrace{
\includegraphics[width=0.15\picwidth]{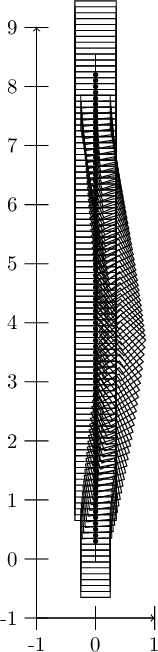}
}_{\SV(\tau_2)}
\underbrace{
\includegraphics[width=0.15\picwidth]{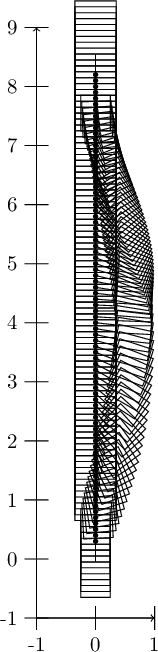}
}_{\SV(\tau_3)}
}_{\text{Workspace } \W = \R^2}$

\caption{Explanatory example of irreducible paths for a $2$-link, $2$-dof
robot, which can move along the $y$-axis, and which has one rotational
joint between its two links, such that its configuration space is $\C =
\R \times [-\frac{\pi}{2},\frac{\pi}{2}]$. {\bf{Left.}} Three
configuration space paths $\tau_1,\tau_2,\tau_3$ with
$\tau_1(0)=\tau_2(0)=\tau_3(0)=q_I$,
$\tau_1(1)=\tau_2(1)=\tau_3(1)=q_G$.  {\bf{Right.}} The workspace volume
of the starting configurations $q_I$, $q_G$, and the swept volume of the
three paths, whereby we have that $\SV(\tau_1) \subset
\SV(\tau_2)$ and $\SV(\tau_1) \subset \SV(\tau_3)$, i.e.  $\tau_2$ and $\tau_3$
are reducible by $\tau_1$, and $\tau_1$ is in fact irreducible. Adapted from
\cite{orthey_2015a}.
\label{irreducible-explanation}}

\end{figure}

%% file: src/irreducible-motion-planning.tex
\section{Irreducible Motion Planning for Serial Kinematic Chains}
\label{mpskc}

As an application, let us approximate the irreducible path space of a serial
kinematic chain. A serial kinematic chain is an alternating sequence of $N+1$
links and $N$ joints as examplified in Fig. \ref{fig:linear_chain}. 

We will call the first link in the chain the \emph{root link} and we will call
the remaining $N$ links \emph{sublinks}. Our assumptions are that all joints are
either revolute or spherical, that the volume of the root link is bigger or
equal to the volume of the sublinks, and that the root link is free-floating.

Under those assumptions, our main idea is the following: if the root link follows a
curvature-constrained path, then the sublinks can be projected into the
swept volume of the root link along the path. Each curvature-constrained path is
thereby associated with an irreducible path. 

Motion planning for a serial kinematic chain is thereby decomposed into two
parts: first, conduct curvature-constrained motion planning for the root link,
and second, project the sublinks into the swept volume of the root link.  We
will first describe how to conduct motion planning for the root link, and then
describe how to construct an algorithm to project the sublinks. 

\subsection{Motion Planning for the Root Link}

The root link is a free-floating rigid body. Our goal is to plan a path for the
root link under a certain maximum curvature constraint $\kappa$. Since the swept
volume of a path is invariant to reparameterization of the path (Sec.
\ref{sec:irrstructure}), we can compute the path where the robot moves at unit
speed.

Planning with a curvature-constrained functional space using constant unit speed
is equivalent to planning a path for a non-holonomic rigid body subject to
differential constraints describing forward non-slipping motions. This is
equivalent to the model of Dubin's car, which can be solved in both 2d and 3d
using kinodynamic planning \citep{lavalle_2006}.

\begin{figure}
        \centering
        \includegraphics[width=0.7\linewidth]{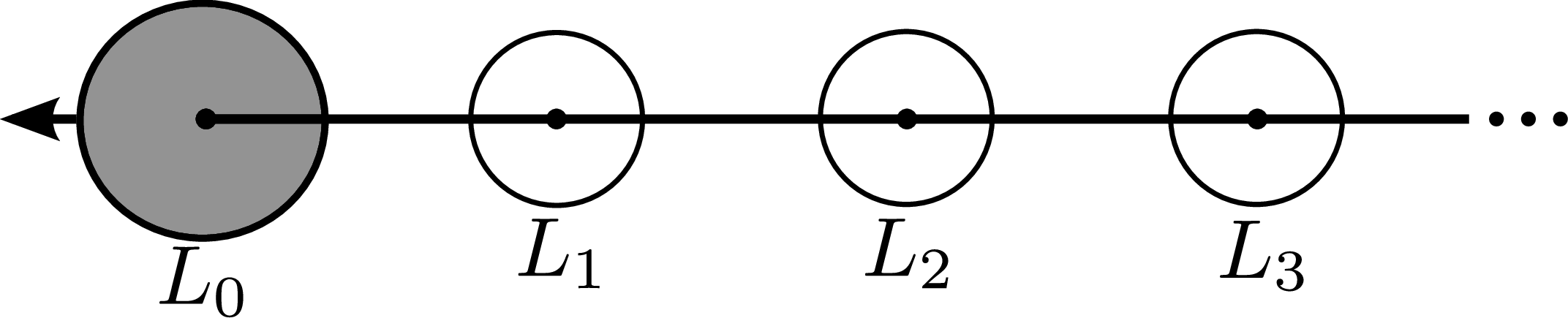}
        \caption{A serial kinematic chain with $L_0$ being the root link, and
                $L_1,L_2,\cdots$ are called the sublinks. 
\label{fig:linear_chain}}
\end{figure}

In 2d, the configuration space of the root link is $SE(2)$ with $q = \left(
x, y, \theta \right)^T$ and the differential model at unit speed is given by

\begin{equation}
  \begin{aligned}
   &\dot{x} = \cos \theta \\
   &\dot{y} = \sin \theta \\
   &\dot{\theta} = u
  \end{aligned}
\end{equation}

\def\atanconstraint{[-\atan(\kappa),\atan(\kappa)]}

where the control space is defined by the steering angle $u \in \atanconstraint$
with $\kappa$ being the curvature constraint. In 3d, the configuration space of
the root link is $\SE(3)$ and the differential model is similar to a driftless
airplane given by 

\begin{equation}
  \begin{aligned}
   \dot{q} = q \left( X_1 + \sum_{i=4}^6 u_i X_i\right) \\
  \end{aligned}
\end{equation}
where 

\begin{equation*}
\begin{aligned}
&X_1= \left[ \begin{smallmatrix} 0 & 0 & 0 & 1 \\ 0 & 0 & 0 & 0 \\ 0 & 0 & 0 & 0 \\ 0 & 0 & 0 & 0 \end{smallmatrix} \right]
&&X_2= \left[ \begin{smallmatrix} 0 & 0 & 0 & 0 \\ 0 & 0 & 0 & 1 \\ 0 & 0 & 0 & 0 \\ 0 & 0 & 0 & 0 \end{smallmatrix} \right]
&&X_3= \left[ \begin{smallmatrix} 0 & 0 & 0 & 0 \\ 0 & 0 & 0 & 0 \\ 0 & 0 & 0 & 1 \\ 0 & 0 & 0 & 0 \end{smallmatrix} \right] \\
&X_4= \left[ \begin{smallmatrix} 0 & -1 & 0 & 0 \\ 1 & 0 & 0 & 0 \\ 0 & 0 & 0 & 0 \\ 0 & 0 & 0 & 0 \end{smallmatrix} \right] 
&&X_5= \left[ \begin{smallmatrix} 0 & 0 & 1 & 0 \\ 0 & 0 & 0 & 0 \\ -1 & 0 & 0 & 0 \\ 0 & 0 & 0 & 0 \end{smallmatrix} \right]
&&X_6= \left[ \begin{smallmatrix} 0 & 0 & 0 & 0 \\ 0 & 0 & -1 & 0 \\ 0 & 1 & 0 & 0 \\ 0 & 0 & 0 & 0 \end{smallmatrix} \right]\nonumber
\end{aligned}
\end{equation*}

%%% bloch page 101 (120), chapter 2: SE(3) 
%%% bloch page 184 (203), chapter 4: basis of lie algebra se(2)

is a basis for $\mathfrak{se}(3)$, the Lie algebra of $\SE(3)$\citep{bloch_2003}.
The controls $u_4,u_5 \in \atanconstraint$ and $u_6 \in \R$ are the yaw, pitch,
and roll steering angles and $X_1$ represents the forward motion at unit speed.

In Appendix \ref{sec:irrlinearchain} we analytically compute the curvature
$\kappa$ for a serial chain in the plane using disk-shaped links. In all our
experiments we have computed $\kappa$ as if the chain would be 2-dimensional. We
have observed that this worked well in experiments. However, further research
needs to investigate the correctness of this claim.

%% file: src/irreducible-algorithm.tex
\subsection{Algorithm to Project Sublinks into Swept Volume of Root Link\label{sec:irralgorithm}}

Let $\tau: [0,1] \rightarrow SE(3)$ be a path for the root link of the serial
chain, and let $\tau$ be constrained by having a maximal curvature $\kappa$ such
that $\kappa(s) \leq \kappa$ for all $s\in[0,1]$. Given $\tau$, let $\SV(\tau)$
be the swept volume of the root link. We claim that for a given $\kappa$, we can
find at least one configuration of the sublinks such that the swept volume of
the sublinks is a subset of $\SV(\tau)$. In Section \ref{sec:correctness} we
discuss the correctness of this claim, and in Appendix \ref{sec:irrlinearchain}
we provide a proof for an N-dimensional serial kinematic chain in the plane.

In this section, we develop the curvature projection algorithm, which takes as
input a path of the root link and provides one configuration of the sublinks.
Our algorithm approximates the serial chain by a set of spheres of radius
$\delta_0,\cdots,\delta_N$ which are connected by lines of length
$l_0,\cdots,l_{N-1}$. The serial chain has $N$ joints centered at each of the
spheres.  Each joint is described by two parameters $(\theta_i,\gamma_i)$,
whereby $\theta_i$ represents the rotation around the normal unit vector
centered at the previous link and $\gamma_i$ represents the rotation around the
binormal unit vector centered at the previous link. The system is then
represented by the position and orientation of the root link in space, i.e.
$SE(3)$, plus its joint configurations $\btheta=\{\theta_1,\cdots,\theta_N\}$
and $\bgamma=\{\gamma_1,\cdots,\gamma_N\}$.

%The algorithm has been implemented in python and is available as a standalone module
% \begin{center}
%         \url{https://github.com/orthez/irreducible-curvature-projection/}
% \end{center}

%In the algorithm we use a $\dt$ parameter to move along the input path. If
%this $\dt$ parameter is too big, then the resulting joint configuration
%will deviate from the true values (we overshoot the true position). We have not
%observed any issues with the parameter in our experiments, but for long chains
%the errors will accumulate, and $\dt$ might need to be decreased to obtain
%robust solutions.

\subsubsection{Algorithmic Description}
\input{src/algorithm-curv-proj.tex}

The resulting algorithm is described in Fig. \ref{algo}. It takes as input the
path of the root link $\tau$, its first and second derivative $\tau'$ and
$\tau''$, the size of the spheres $\delta_{0:N}$ for the root link and the $N$
sublinks, and the length of the links $l_{1:N}$. It outputs the configurations
of the sublinks $\theta_{1:N},\gamma_{1:N}$, such that each sublink is inside
the swept volume of the root link. We assume that there exists a world frame $O$
with basis $\x,\y,\z$. Starting from $\tc=0$ we compute a frame $S_0$ centered
at the root link with orthonormal basis $\ex,\ey,\ez$ (Line 1-4), and we compute
the rotational transformation matrix $\Rab$ between $O$ and $S_0$ (Line 5). Then
for each sublink (Line 6), we start at $\tc$, and we follow $\tau$ backwards
until the distance between $\tau(\tn)$ and $\tau(\tc)$ is equal to $l_i$ (Line
7-9). This position marks the position of the $i$-th sublink. We mark the
position as $\tn$ (Line 10), we compute the vector from $S$ to $\tn$ (Line 11),
and we rotate this vector into the world frame $O$ (Line 12). Then we compute
the angle of the vector to the $xy$ and the $zx$ plane, respectively (Line
13-21). Those angles give the configuration $\theta_i,\gamma_i$ of sublink $i$.
Finally, we rotate the rotation matrix $\Rab$ correspondingly (Line 22) to
obtain a new frame $S_i$ centered at link $i$ (Line 23-26). The algorithm is
iterated until all $N$ sublinks have been placed in that manner. See also Fig.
\ref{fig:irrtraj} for a visualization of the algorithm in a 2D setting. 

\begin{figure}[h]
        \centering
        \includegraphics[width=0.8\linewidth]{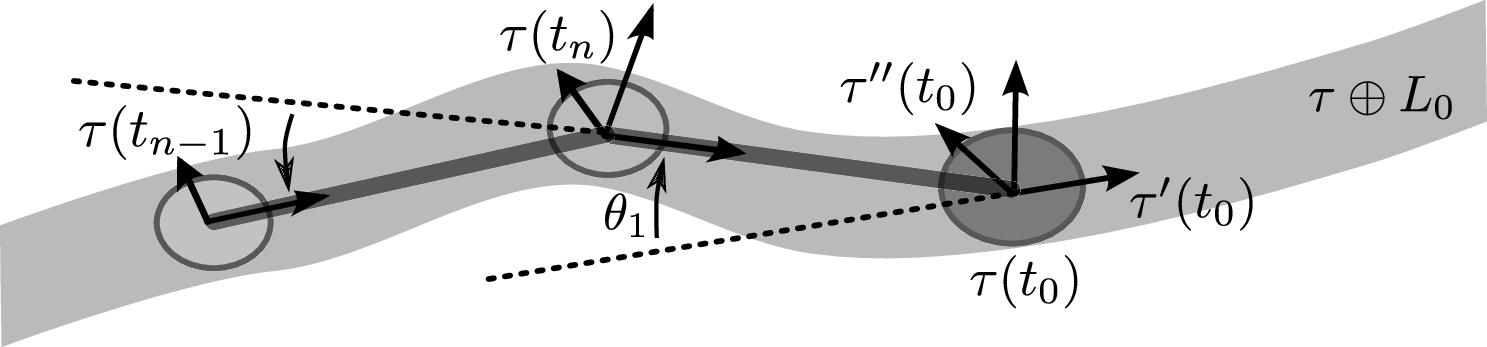}
        \caption{Given a path $\tau\in\Fkn$, we can analytically compute
        the joint configurations, such that sublinks of the serial kinematic chain are
reduced, i.e. they are inside of the swept volume of $\tau \oplus L_0$.\label{fig:irrtraj}}
\end{figure}

\subsubsection{Complexity and Correctness \label{sec:correctness}}

The complexity of the algorithm is $\mathcal{O}(N)$, $N$ being the number of
sublinks. In Appendix \ref{sec:irrlinearchain} we prove the correctness of the algorithm
for a serial kinematic chain in 2D, whereby we assume that the lengths between
joints is equidistant. The proof first verifies the correctness of the algorithm
of a chain with $N=1$ sublinks and then generalizes this result to arbitrary
sublinks $N>1$. The proof of correctness of the algorithm with arbitrary lengths
and with spherical joints in 3D is subject of further research.

%% file: src/algorithm-curv-proj.tex
%\begin{figure}[h!]
\begin{algorithm}%[ht]
\KwData{$\tau, \tau', \tau'', \delta_{0:N}$, $l_{1:N}, \Dt$}
\KwResult{$\theta_{1:N},\gamma_{1:N}$}
$\ex \leftarrow \tau'(0)$\;
$\ey \leftarrow \tau''(0)$\;
$\ez \leftarrow \tau'(0) \times \tau''(0)$\;

 $\tc \leftarrow 0$\;

 $\Rab \leftarrow \begin{pmatrix}
         \ex\cdot\x & \ey\cdot\x & \ez\cdot\x\\
         \ex\cdot\y & \ey\cdot\y & \ez\cdot\y\\
         \ex\cdot\z & \ey\cdot\z & \ez\cdot\z
 \end{pmatrix}$\;

 \For{ $i \leftarrow 1$ \KwTo $N$ }{

         $\tn \leftarrow \tc$\;
         \While{$\|\tau(\tn) -\tau(\tc)\| \leq l_i$}{
                 $\tn \leftarrow \tn-\Delta t$
         }
         $\tau_n \leftarrow \tau(\tn)$\;

         $\pb\leftarrow \tau(\tn) - \tau(\tc)$\;
         $\pa \leftarrow \Rab^T\pb$\;
         $\xl \leftarrow (-1,0,0)^T$\;

         $\pxy \leftarrow \pa - (\pa^T\ze)\ze$\;
         $\pzx \leftarrow \pa - (\pa^T\ye)\ye$\;
         $\theta_i \leftarrow \acos( \frac{\pxy^T\xl}{\|\pxy\|\|\xl\|})$\;
         $\gamma_i \leftarrow \acos( \frac{\pzx^T\xl}{\|\pzx\|\|\xl\|})$\;
        
        \If{$\pa^T\ze < 0$}{
                $\gamma_i \leftarrow -\gamma_i$\;
        }
        \If{$\pa^T\ye > 0$}{
                $\theta_i \leftarrow -\theta_i$\;
        }
        
        $\Rab \leftarrow \Rab\cdot \Ry(\gamma_i)\cdot \Rz(\theta_i)$\;

        $\ex \leftarrow \Rab \xe$\;
        $\ey \leftarrow \Rab \ye$\;
        $\ez \leftarrow \Rab \ze$\;
         $\tc \leftarrow \tn$\;

 }
 %\caption{Irreducible Curvature Projection}
\caption{Irreducible Curvature Projection Algorithm \label{algo}}
\end{algorithm}
%\end{figure}

%% file: src/experiments.tex
\section{Experiments\label{sec:experiments}}

We will show that the concept of an irreducible path space can be applied to any
motion planning algorithm taking curvature constraints into account, and that
each planning algorithm using the irreducible path space outperforms the same
motion planning algorithm using the space of all continuous paths.

We performed seven experiments to test our hypothesis. Our first three
experiments are planning scenarios for an idealized serial kinematic chain in a
2d maze, a 2d rock environment and a 3d rock environment. We compared the
kinodynamic {\bf{R}}apidly-exploring {\bf{R}}andom-{\bf{T}}ree (\rrt)
\citep{lavalle_2001} algorithm both using the original space and the irreducible
path space. In the fourth experiment we plan a path for a mechanical snake in a
turbine environment. We compared four different motion planning algorithms:
\rrt, {\bf{P}}ath-{\bf{D}}irected {\bf{S}}ubdivision {\bf{T}}ree planner (\pdst)
\citep{ladd_2004}, {\bf{K}}inodynamic motion {\bf{P}}lanning by
{\bf{I}}nterior-{\bf{E}}xterior {\bf{C}}ell {\bf{E}}xploration (\kpiece)
\citep{sucan_2009} and {\bf{S}}table {\bf{S}}parse {\bf{T}}ree (\sst)
\citep{li_2016}. In the fifth experiment we plan a path for a mechanical octopus
in a pipe environment, using the same four algorithms. In all those experiments
we assume that the head of the robot can be independently controlled. In the sixth and seventh experiments we plan a path for a
humanoid robot, first in a room with doors of different heights and second on a
floor with a hole in a wall.

\begin{figure*}
        \centering
        \def\figWidth{0.24\linewidth}
        \def\figH{0.25\linewidth}
        \includegraphics[width=\figWidth,height=\figH]
        {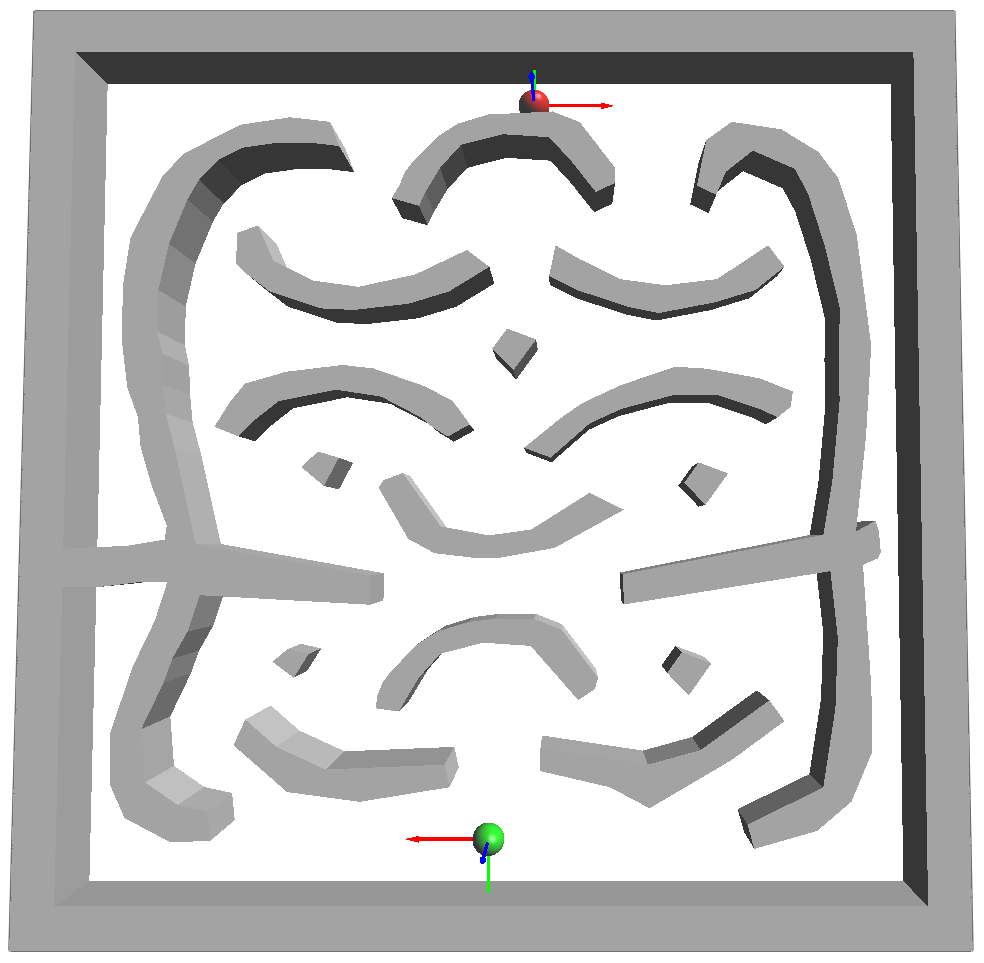}
        \includegraphics[width=\figWidth,height=\figH]
        {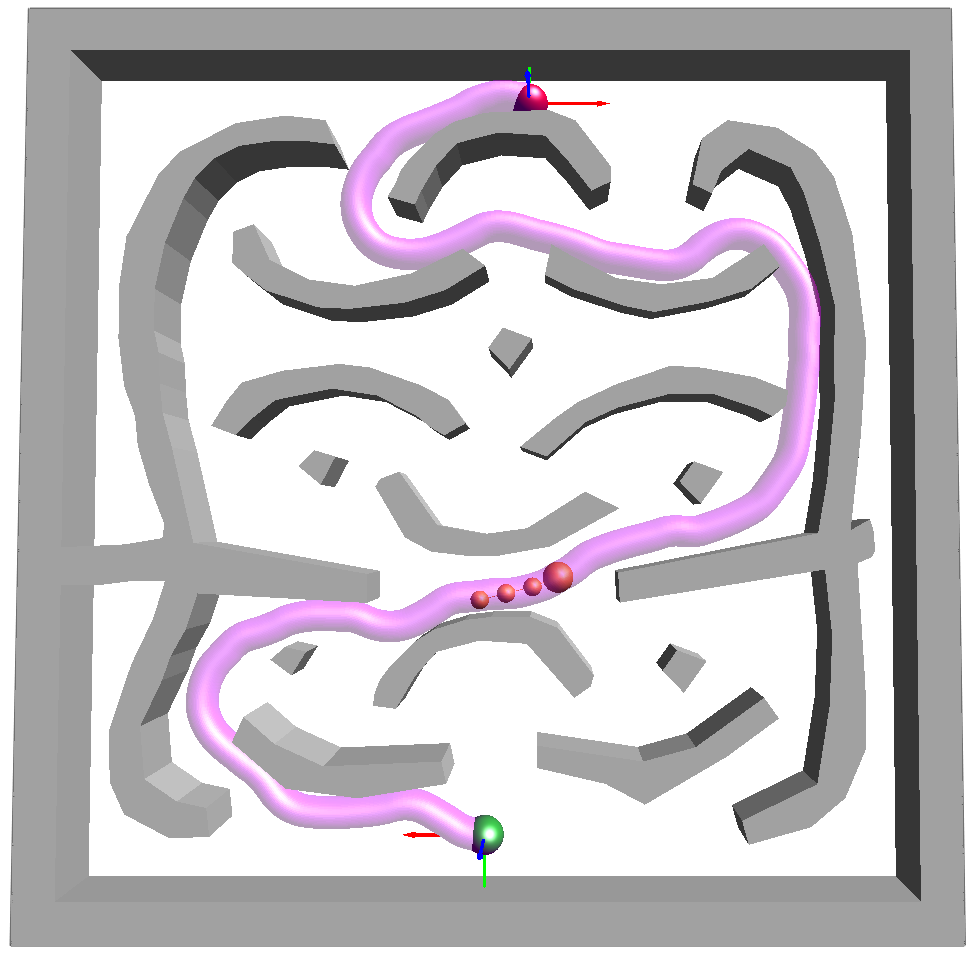}
        \includegraphics[width=\figWidth,height=\figH]
        {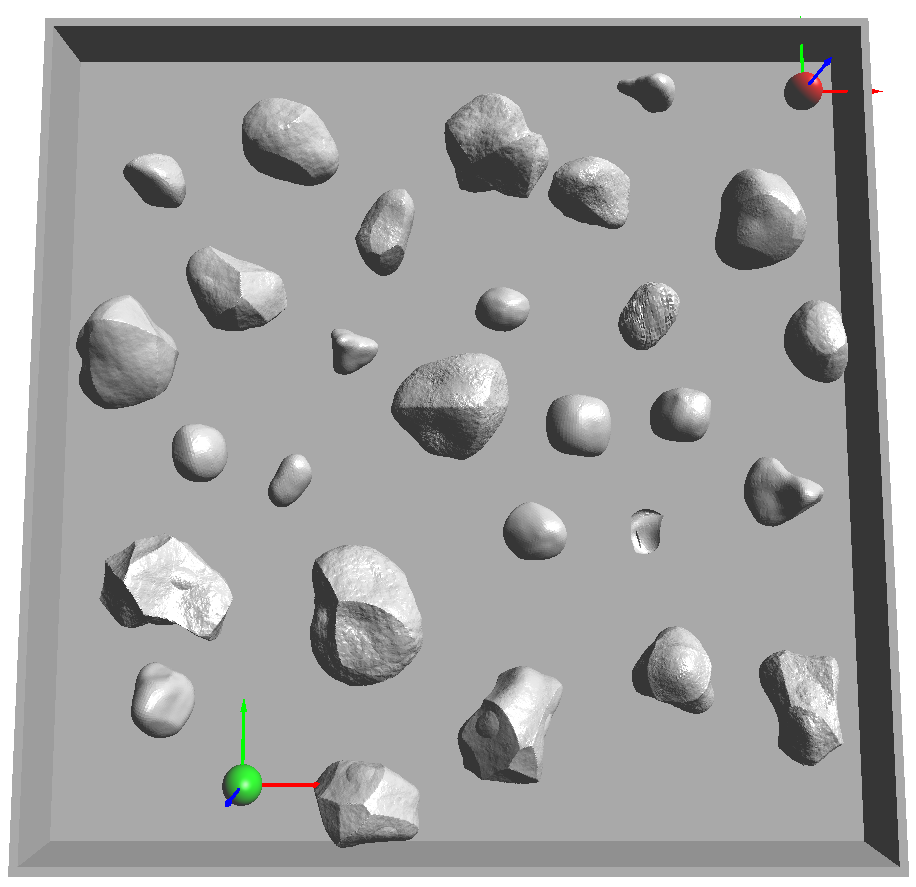}
        \includegraphics[width=\figWidth,height=\figH]
        {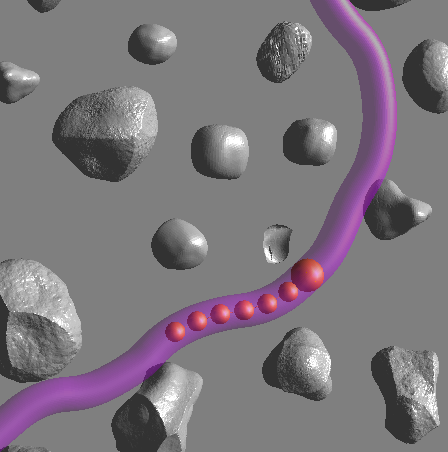}
        \caption{Serial kinematic chain in 2D (From Left to Right). Image 1 and
        2 show the start and goal configuration of Experiment $1$ and the swept
      volume along one irreducible solution path, respectively. Image 3 and 4
    show a serial kinematic chain in an environment filled with rocks. We show
  the start and goal configuration and the swept volume along a solution path,
respectively.\label{fig:maze}}
\end{figure*}

%In each experiment we report the average runtime of the motion planning
%algorithm. In the case of using the irreducible path space, we additionally need
%to project the sublinks into the swept volume of the root link.  This
%computation requires less than $0.2$ seconds for the mechanical octopus, less
%than $0.01$ seconds for the mechanical snake and less than $0.05$ seconds for
%the arms of the humanoid robot. These computations are negligible and have not
%been included in the reported runtime. All experiments have been conducted using
%the open motion planning library (\ompl) \citep{sucan_2012}.

\def\epsgoal{\epsilon_{\text{goal}}}

For each experiment we specify the values of the serial kinematic chain, the
maximum curvature $\kappa$, the joint limit $\theta^L$, the size of the root
link $\delta_0$ and the number of sublinks $N$. Each experiment is repeated $M$
times and it is terminated if a time threshold $T$ is reached or if a goal
region of size $\epsgoal$ around the goal configuration is reached.

\begin{figure*}
  \centering
  \def\hh{0.25}
  \includegraphics[width=0.39\linewidth,height=\hh\linewidth]
  {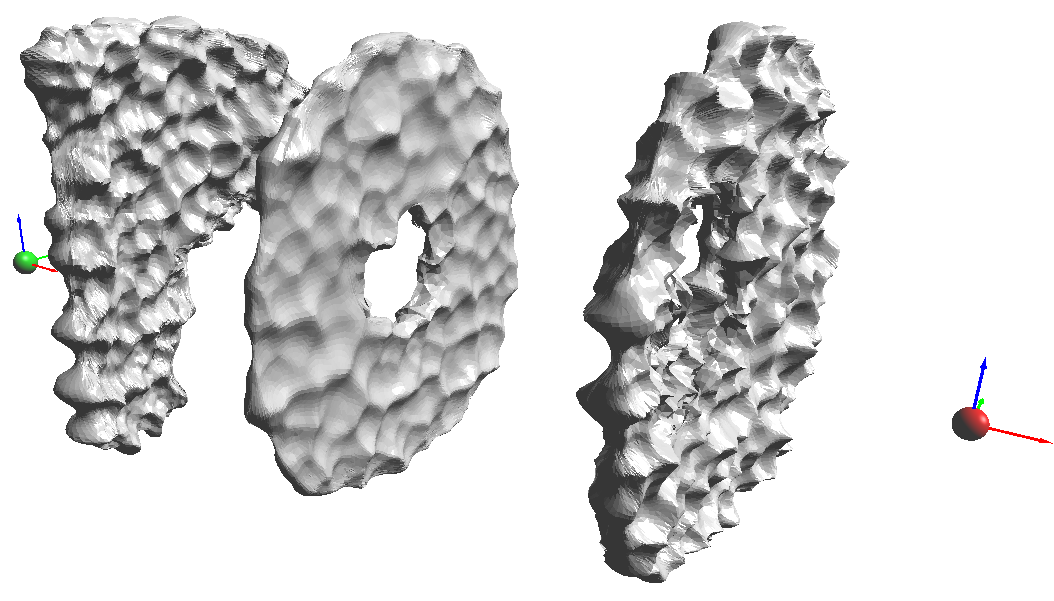}
  \includegraphics[width=0.39\linewidth,height=\hh\linewidth]
  {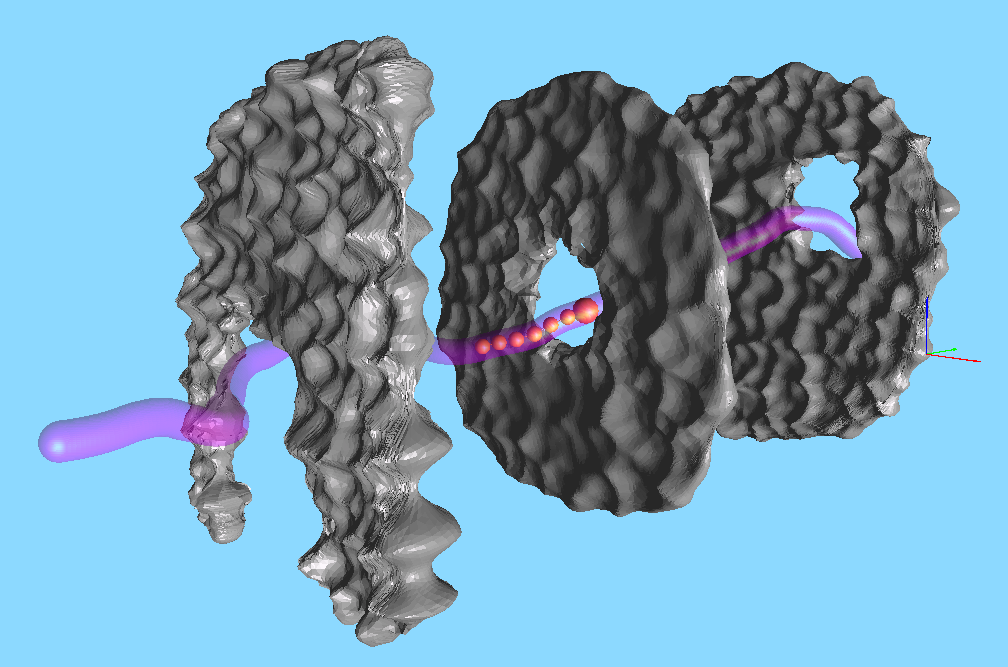}
  \includegraphics[width=0.19\linewidth,height=\hh\linewidth]
  {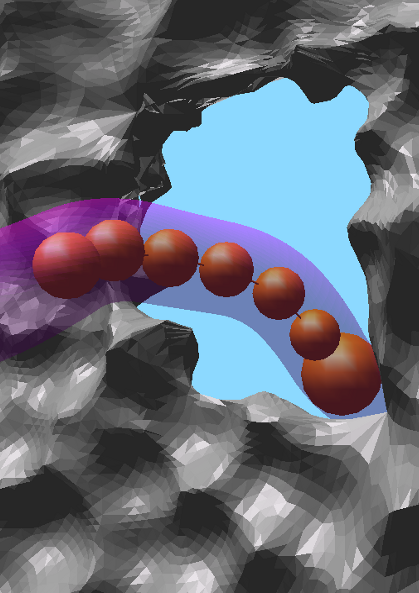}

  \caption{Planning for the root link of a serial kinematic chain on $\SE(3)$. {\bf Left:}
    3d rocks environment with starting position (green) and goal position (red).
    {\bf Middle:} An irreducible path found by \rrt [Irreducible]. The swept
    volume of the root link is shown in magenta. The
    position of the sublinks is an output of the curvature projection algorithm.
  {\bf Right:} Close-up of a position of the robot along the irreducible path,
showing how the sublinks are inside the swept volume of the root link.\label{fig:swimming-snake}}
\end{figure*}

\begin{figure*}
  \centering
  \def\ww{0.23} 

  \def\hh{0.2}

  \includegraphics[width=\ww\linewidth,height=\hh\linewidth]
  {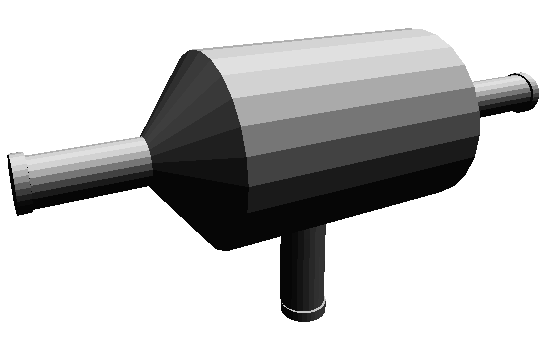}
  \includegraphics[width=\ww\linewidth,height=\hh\linewidth]
  {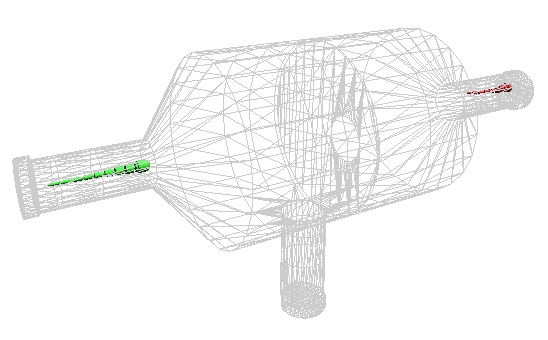}
  \includegraphics[width=\ww\linewidth,height=\hh\linewidth]
  {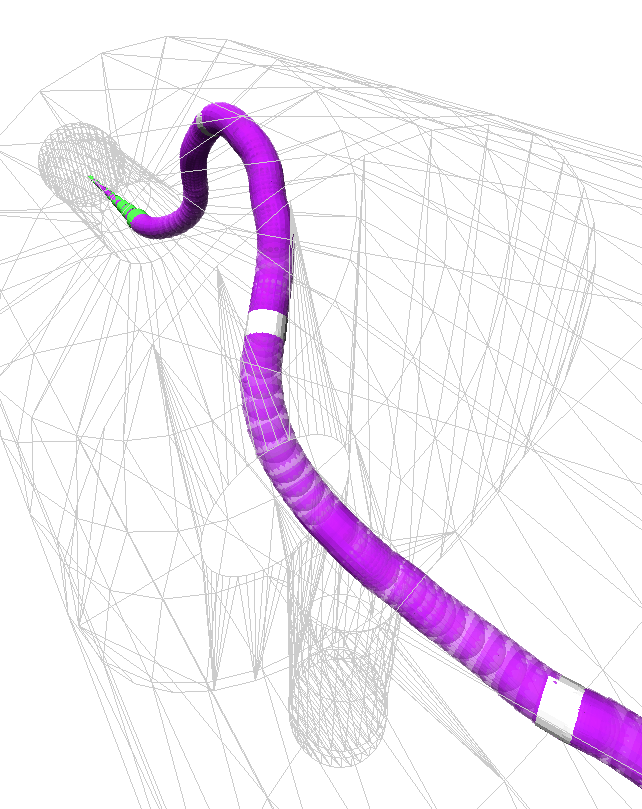}
  \includegraphics[width=\ww\linewidth,height=\hh\linewidth]
  {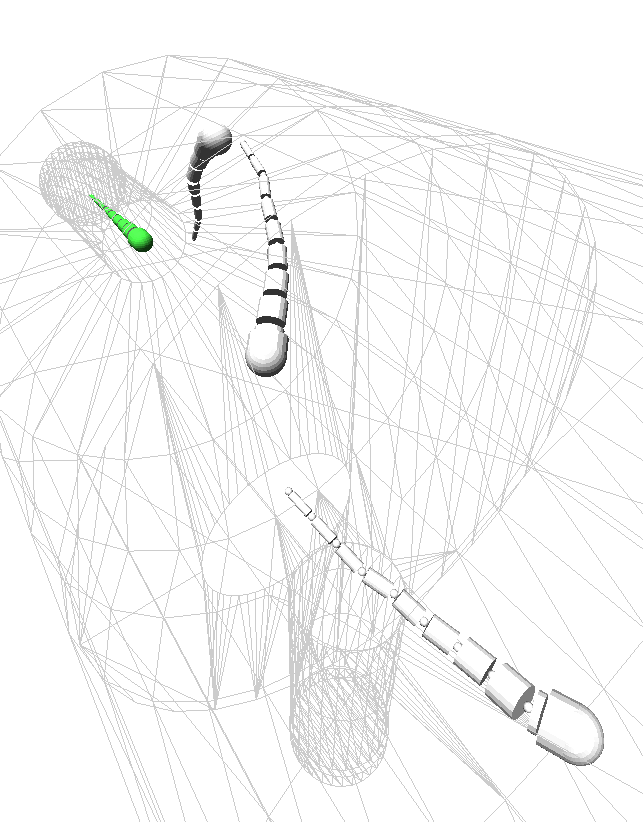}
  
  \def\hh{0.22}
  \includegraphics[width=\ww\linewidth,height=\hh\linewidth]
  {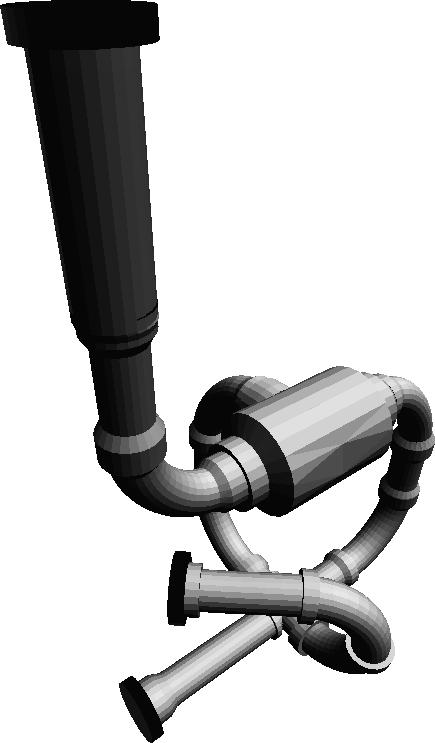}
  \includegraphics[width=\ww\linewidth,height=\hh\linewidth]
  {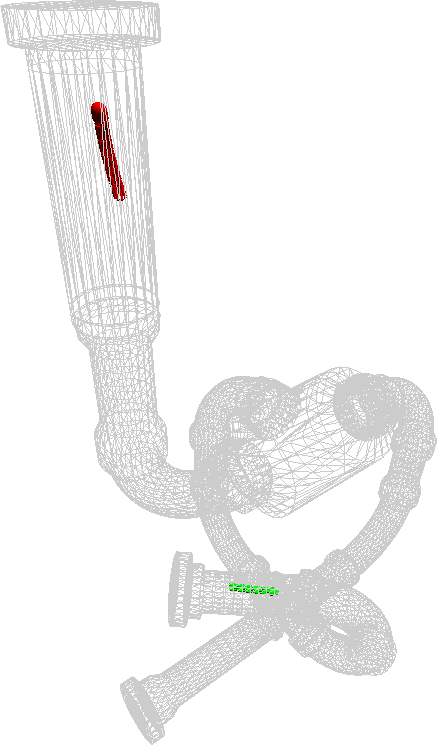}
  \def\hh{0.2}
  \includegraphics[width=\ww\linewidth,height=\hh\linewidth]
  %{images/klamptompl/pipecloseup1.png}
  {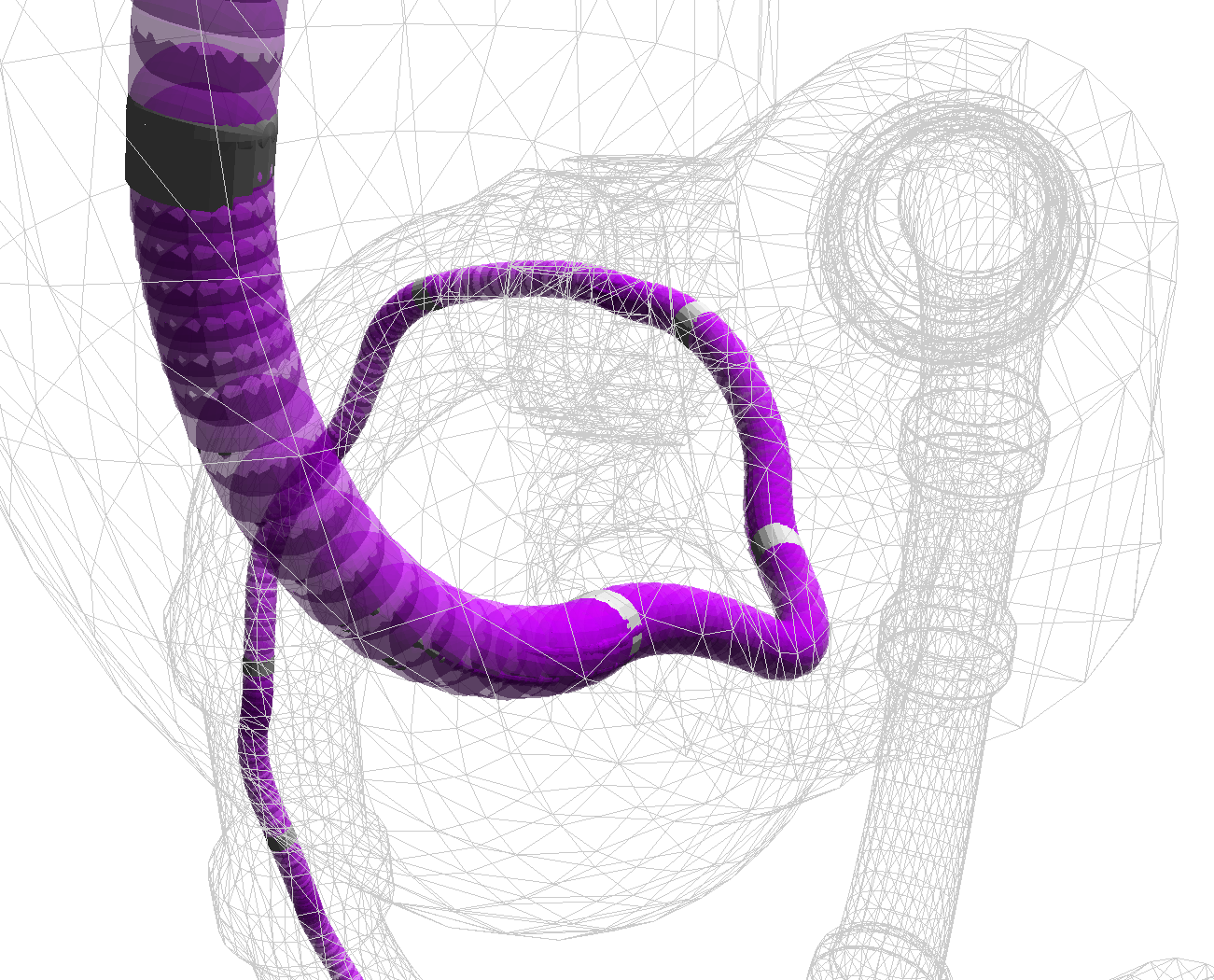}
  \includegraphics[width=\ww\linewidth,height=\hh\linewidth]
  %{images/klamptompl/pipecloseup2.png}
  {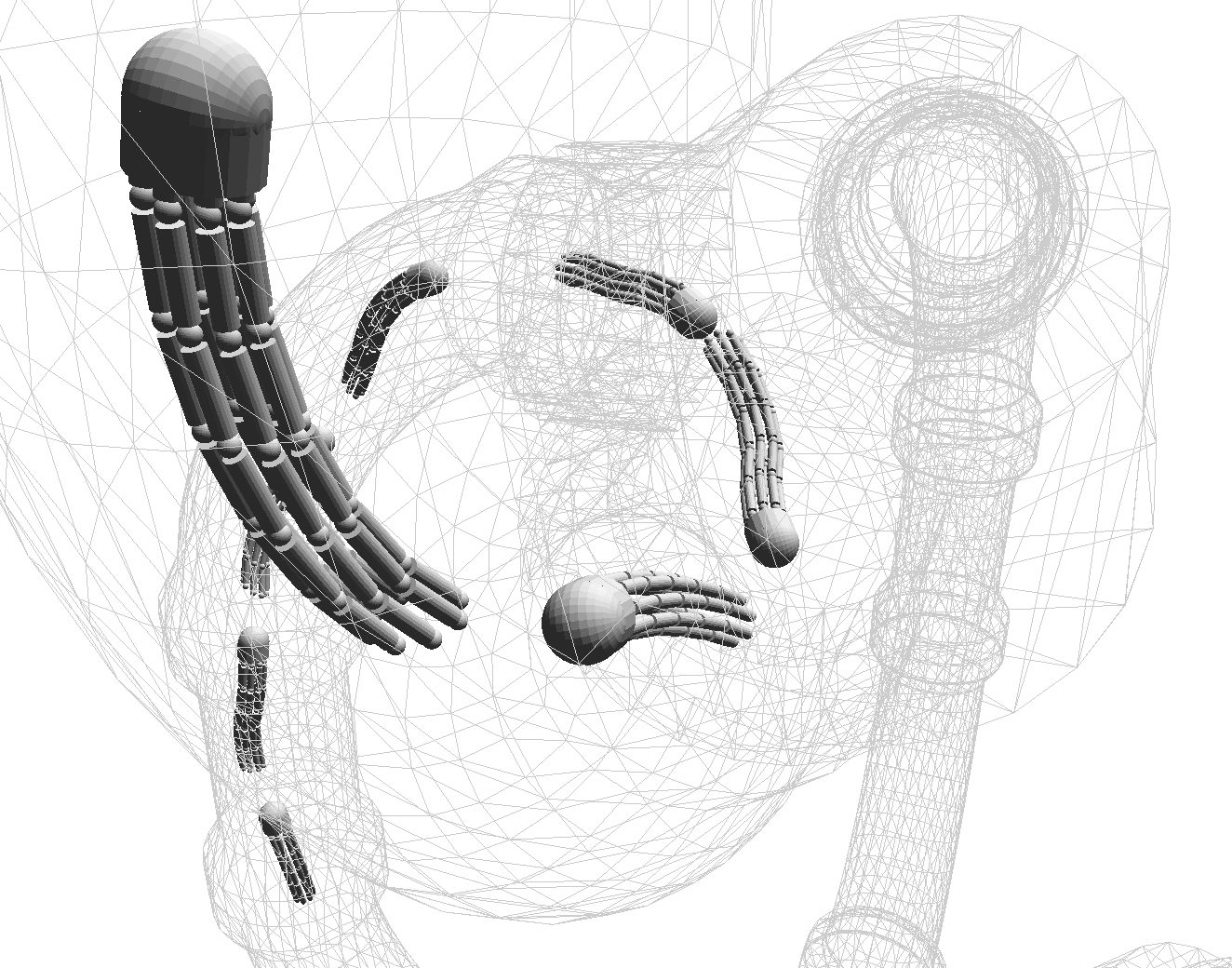}

  \def\hh{0.28}
  \includegraphics[width=\ww\linewidth,height=\hh\linewidth]
  {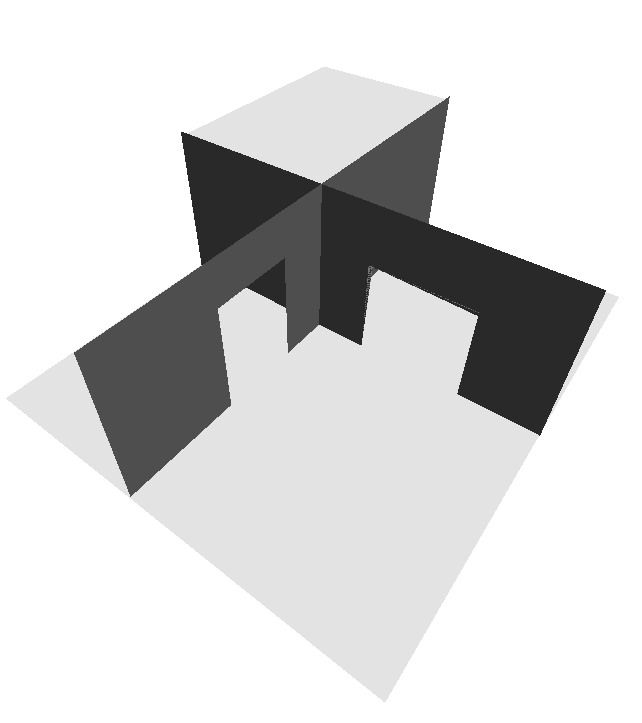}
  \includegraphics[width=\ww\linewidth,height=\hh\linewidth]
  {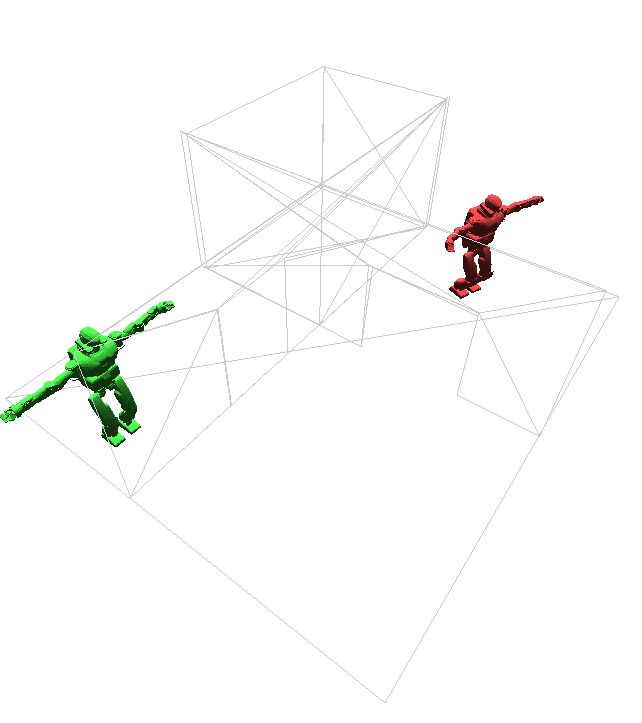}
  \includegraphics[width=\ww\linewidth,height=\hh\linewidth]
  {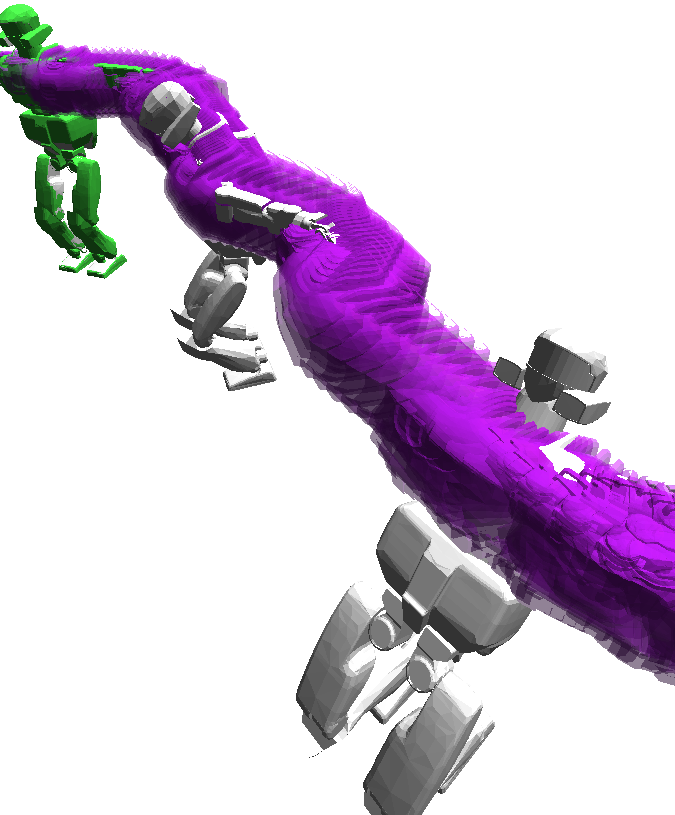}
  \includegraphics[width=\ww\linewidth,height=\hh\linewidth]
  {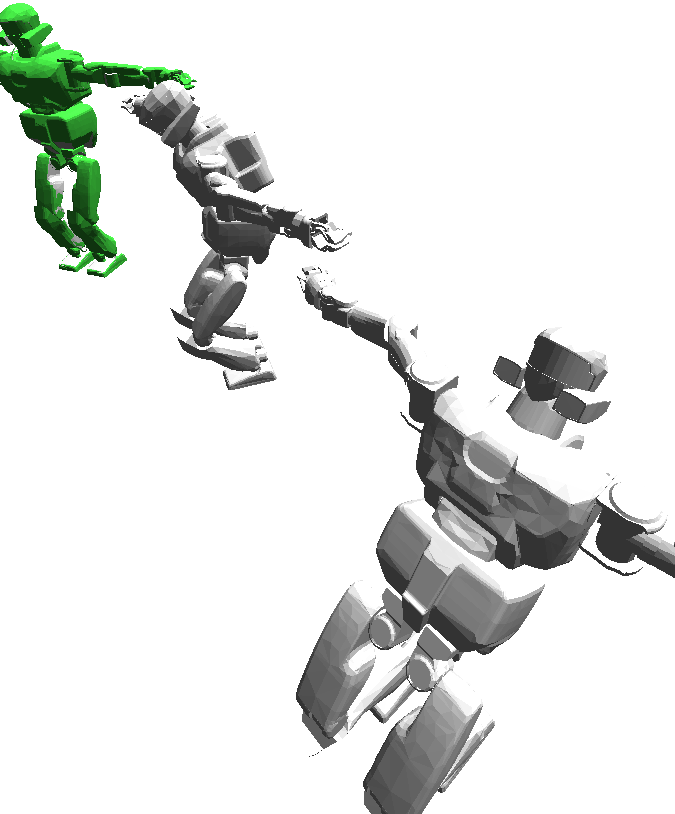}

  \def\hh{0.25}
  \includegraphics[width=\ww\linewidth,height=\hh\linewidth]
  {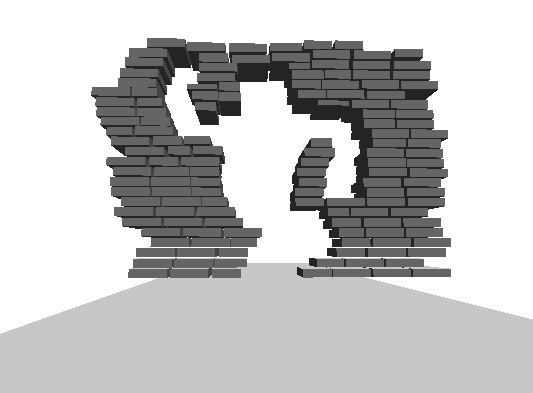}
  \includegraphics[width=\ww\linewidth,height=\hh\linewidth]
  {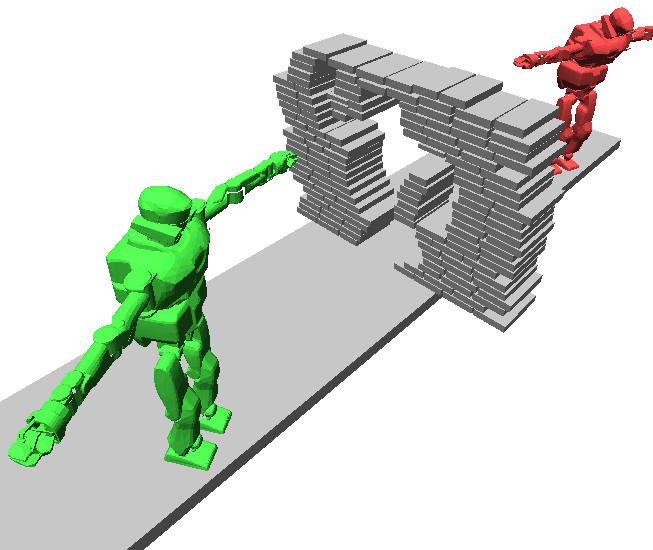}
  \def\hh{0.28}
  \includegraphics[width=\ww\linewidth,height=\hh\linewidth]
  {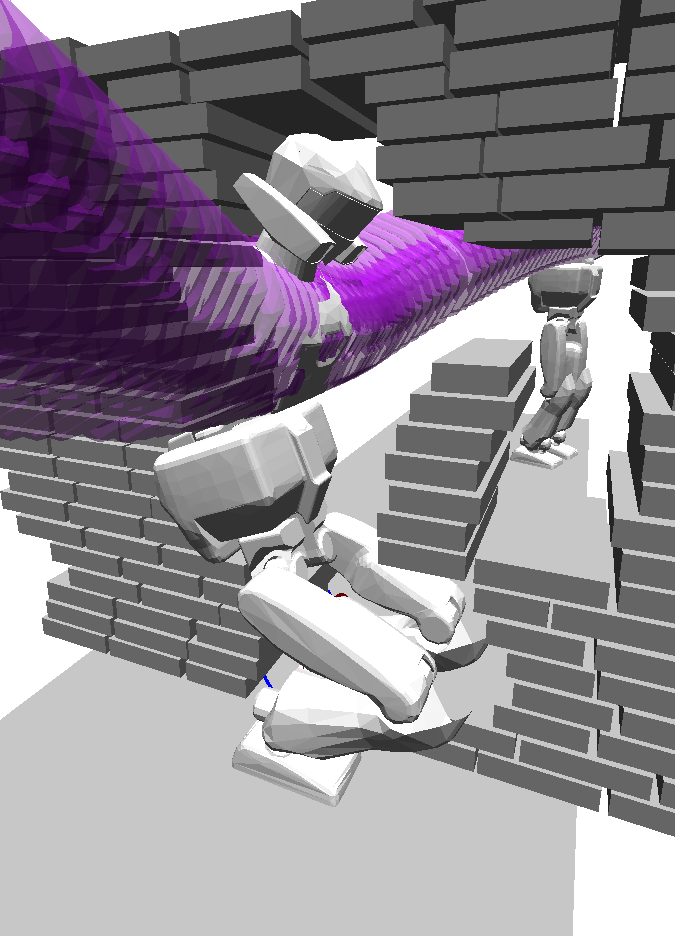}
  \includegraphics[width=\ww\linewidth,height=\hh\linewidth]
  {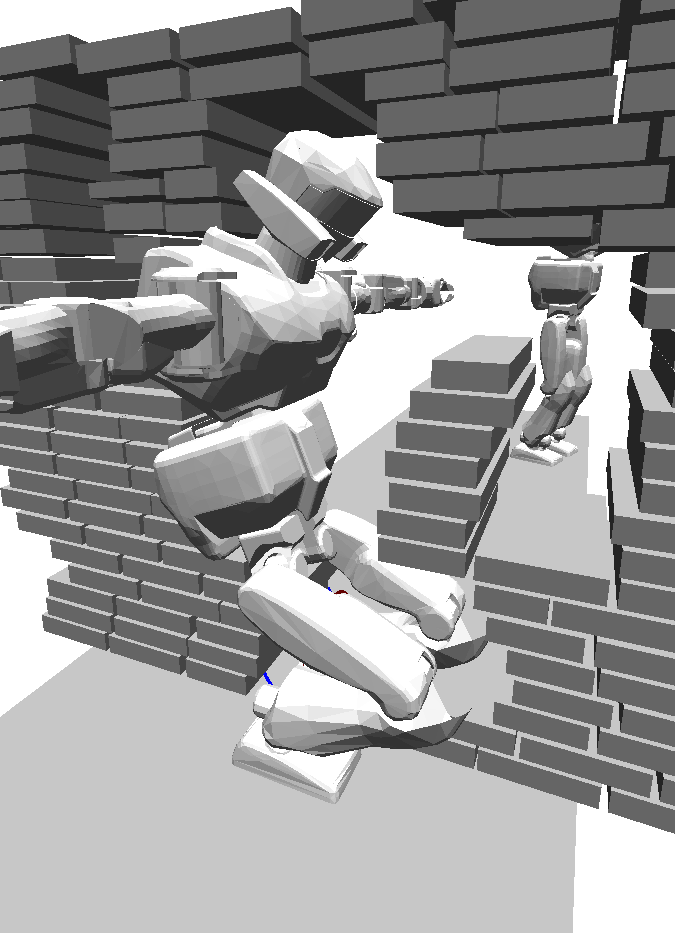}

  \def\ww{0.22} \def\hh{0.3}

  \caption{Visualization of the motion planning experiments. For each experiment we show the
    environment (column 1), the start configuration in green and the goal
    configuration in red (column 2), the close up of a swept volume of one irreducible
    solution path in magenta (column 3) and
    milestones along the swept volume in gray (column 4). The first row (Experiment 4)
    shows a mechanical snake in a turbine environment. The second row
    (Experiment 5)
    is a mechanical octopus in a pipe environment. The third row (Experiment 6)
    shows a humanoid robot moving sideways through a room with doors of
    different height. For better visualization, only the swept volume of the
    chest is visualized. The fourth row (Experiment 7) shows the same humanoid
    robot moving sideways through a hole in a wall shaped according to the
  geometry of the robot.\label{fig:experiments}}

\end{figure*}
\subsection{Experiment 1: Serial Kinematic Chain in 2D maze}

Our first experiment is a 2d maze environment as depicted in Fig.
\ref{fig:maze}, where a serial kinematic chain has to be moved from a given
start to a given goal configuration. We assume that the root link can be
independently actuated. The configuration space is
$\SE(2) \times \R^3$. We compare \rrt using the full space of paths with \rrt
\irreducible using the space of irreducible paths. We report on the success
rate, the average time to plan, and the standard deviation of the planning
algorithm in Tab. \ref{table:results}. It can be seen that \rrt \irreducible has
a lower planning time of one order of magnitude.
The parameters used were $\kappa = 1$, $\theta^L = \frac{\pi}{2}$,
$\delta_0 = 0.23$, $N = 3$, $M=100$, $T=3600$s and $\epsgoal = 0.1$.

%$l_0 = 0.33$, $\theta^L = \frac{\pi}{2}$, $\delta_0 = 0.23$,
%$\delta_i = 0.138$, $N = 3$ and $\epsilon_{dist to goal} = 0.1$.

\subsection{Experiment 2: Serial Kinematic Chain in 2D rock environment}

Our second experiment is a 2d rock environment as depicted in Fig.
\ref{fig:maze}. We compared again \rrt with \rrt \irreducible. The results are
reported in Table \ref{table:results} and show that \rrt \irreducible using the
irreducible path space outperforms \rrt using the space of continuous paths. The
parameters used were $\kappa = 1$, $\theta^L = \frac{\pi}{2}$, $\delta_0 =
0.23$, $N = 6$, $M=100$, $T=3600$s and $\epsgoal = 0.1$.

%parameters used were $\kappa = 1$, $l_0 = 0.33$, $\theta^L = \frac{\pi}{2}$,
%$\delta_0 = 0.23$, $\delta_i = 0.138$, $N = 6$ and $\epsilon_{dist to goal} =
%0.1$.

\subsection{Experiment 3: Serial Kinematic Chain in 3D rock environment}

Our third experiment changes the 2d rock environment into a 3d rock environment,
where the serial kinematic chain has to move through a series of holes to reach a
target. The configuration manifold is $\SE(3)\times\R^{12}$.

We compare again \rrt and \rrt \irreducible, results shown in Table
\ref{table:results}.  Fig. \ref{fig:swimming-snake} shows a time instance from
one successful run of the \rrt \irreducible, where the swept volume of the
planned motion for the head is shown in magenta, and the sublinks are shown as
the results of our projection algorithm. Parameters used were $\kappa = 1$,
$\theta^L = \frac{\pi}{2}$, $\delta_0 = 0.23$, $N = 6$, $M=100$, $T=3600$s and
$\epsgoal = 0.1$.

\subsection{Experiment 4: Mechanical snake in 3D turbine environment}

In the fourth experiment we use a mechanical snake which has to move through a
3D turbine environment. Inside the turbine there is a narrow hole through which
the snake has to move. The dimensionality of the configuration space is $\SE(3)
\times \R^{16}$. We compare \rrt, \pdst, \kpiece and \sst with the full path
space against the same algorithms using the irreducible path space. The results
in Table \ref{table:results} show that each original algorithm is outperformed
by the same algorithm using the irreducible path space. The overall best average
computation time has been achieved by \kpiece \irreducible. The environment, the
start and goal configuration, a swept volume along the one irreducible solution
path and a close up of milestones are shown in row $1$ of Fig.
\ref{fig:experiments}.  The parameters are $\kappa = 1.57$, $\theta^L =
\frac{\pi}{4}$, $\delta_0=0.1$, $N=8$, $M=100$, $T=1200$s and $\epsgoal=0.5$.

%The parameters are $\kappa =
%\dfrac{2\sin(\frac{\pi}{4})}{9\cdot 0.1} \approx 1.57$, $l_i = 0.14$,
%$\delta_0=0.1$, goal set $\epsilon=0.5$, $N=8$, $M=100$, $T=1200$, $\theta^L =
%\frac{\pi}{4}$

\subsection{Experiment 5: Mechanical octopus in 3D pipe environment}

In the fifth experiment we use a mechanical octopus which has $8$ arms with each
$5$ sublinks leading to a configuration space of dimensionality $\SE(3) \times
\R^{80}$.  For the irreducible path space we compute a path for the head on
$\SE(3)$, then project all the remaining links into the swept volume of the head
by applying the curvature projection algorithm on each arm individually. Results
in Table \ref{table:results} indicate that the original problem was too
difficult to be solvable by any algorithm. However, the irreducible path space
variations can find a solution with \pdst \irreducible achieving the best
average computation time while succeding in $100$ percent of the cases. Row $2$
of Fig. \ref{fig:experiments} visualizes the environment and one solution path.
The parameters used were $\kappa = 2.66$, $\theta^L = \frac{\pi}{2}$,
$\delta_0=0.1$, $N=5$, $M=100$, $T=1200$, and $\epsgoal=1.0$.

%The parameters used were $\kappa =
%\dfrac{2\sin(\frac{\pi}{2})}{6\cdot 0.15} \approx 2.66$, $l_i = 0.19$,
%$\delta_0=0.1$, goal set $\epsilon=1.0$, $M=100$, $T=1200$, and $\theta^L =
%\frac{\pi}{2}$. 
%The parameters are $\kappa = 1.57$, $\theta^L =
%\frac{\pi}{4}$, $\delta_0=0.1$, $N=8$, $M=100$, $T=1200$s and $\epsgoal=0.5$.

\subsection{Experiment 6: Humanoid Robot in Room environment}

In the sixth experiment we consider motion planning for the sideways motion of a
humanoid robot as shown in row $3$ of Fig. \ref{fig:experiments}. This can be
helpful to estimate if a humanoid robot can potentially fit through a door or a
small opening. The environment consists of two doors with different heights. We
make the assumption that the robot slides on the planar floor leading to the
configuration space $\SE(2) \times \R^{19}$.  We apply the idea of the
irreducible path space to the chest of the robot, such that the arms of the
robot behave like the sublinks of the serial kinematic chain. Since each arm has
$7$ dofs, the resulting dimensionality is $\SE(2) \times \R^{5}$.  As shown in
Table \ref{table:results} each algorithm using the irreducible path space
outperforms the same algorithm using the full path space. The best algorithm in
terms of average computation time is \rrt \irreducible achieving $100$ percent
success rate. A solution path is shown in row $3$ of Fig.
\ref{fig:experiments}, showing how the arms have been projected into the swept
volume of the chest. The parameters are $\kappa = 2.66$, $\theta^L =
\frac{\pi}{2}$, $\delta_0=0.1$, $N=3$, $M=100$, $T=1200$, and $\epsgoal = 0.5$.

%parameters are $\kappa = \dfrac{2\sin(\frac{\pi}{2})}{0.75} \approx
%2.66$, $l_i = 0.25$, $\delta_0=0.1$, goal set $\epsilon=0.5$, $M=100$, $T=1200$,
%and $\theta^L = \frac{\pi}{2}$.

\subsection{Experiment 7: Humanoid Robot in Hole in Wall environment}

The last experiment is similar to experiment $6$ with a more challenging
environment. The humanoid robot has to move through a hole in a wall which is
shaped according to the robot's geometry. This is difficult , since a possible
solution path has to overcome the narrow passage in the configuration space. Due
to this difficulty we only used the best algorithm from experiment $6$, the \rrt
using $M=10$ runs with a timelimit of $T=86400$s or $24$h. We compared the
performance of \rrt \irreducible with \rrt as shown in Table
\ref{table:results}. It can be seen that \rrt \irreducible is able to find a
path although it takes on average $100$ minutes to obtain a solution. \rrt was
not able to find a solution in the given time limit. A solution path is
shown in row $4$ of Fig. \ref{fig:experiments} after a shortcut procedure was
applied. This experiment is a reimplementation of the experiment conducted in
\cite{orthey_2015a}.

%\subsection{Limit Case}
%
%Since the curvature of the paths are restricted, there are environments which
%cannot be solved using the space of irreducible paths. One example is a twisted
%cylinder environment as depicted in Fig. \ref{fig:twister}. The curvature of the
%twist is higher than the allowable curvature of the mechanical snake. However, a
%standard \rrt was able to find a solution after $160$ seconds.
%
%\begin{figure}
%  \centering
%  \def\hh{0.3}
%  \def\ww{0.15} 
%  \includegraphics[width=\ww\linewidth,height=\hh\linewidth]
%  {images/klamptompl/twister1-crop.png}
%  \def\ww{0.25} 
%  \includegraphics[width=\ww\linewidth,height=\hh\linewidth]
%  {images/klamptompl/twister2-crop.png}
%  \includegraphics[width=\ww\linewidth,height=\hh\linewidth]
%  {images/klamptompl/twister3-crop.png}
%  \caption{A motion planning problem for a mechanical snake in a twisted
%  cylinder environment. Start configuration in green, goal configuration in red.
%Due to the approximation of the irreducible path space, this environment cannot
%be solved by our method since the curvature of the space is larger than the
%maximum curvature of the path space. However, a standard \rrt solved the problem
%after $160s$. Milestones along the path are shown.\label{fig:twister}}
%\end{figure}

%% file: src/experiments-results.tex
\def\Vhrulefill{\leavevmode\leaders\hrule height 0.7ex depth \dimexpr0.4pt-0.7ex\hfill\kern5pt}
\begin{table}%[!htbp]
%\begin{adjustbox}{angle=90}
\centering
\begin{center}
        \caption{
                Results of the seven experiments. Each algorithm used is
                compared between its original version and the irreducible
                version using the irreducible path space.
        \label{table:results}
        }
        %\scriptsize
        \fontsize{0.3cm}{0.35cm}\selectfont

        \renewcommand{\arraystretch}{1.1}
        \newcommand\titleC{\cellcolor{gray!10}}
        \newcommand\bc[1]{\color{black!60}{#1}}
        \def\CG{\ensuremath{\mathbb{S}}\xspace}
        \newcommand\insertTableCellHead[1]{
          \hline
          \multicolumn{9}{|c|}{
          \raisebox{-1ex}[0ex][0ex]{\titleC #1}}\\[1.8ex]
          \hline
          %\rowcolor{black!5}Algorithm & Manifold & Success($\%$) & \multicolumn{5}{c}{Time($s$)}\\
          %\hline
        }

        \newcommand\insertTableCellWin[7]{  #1 &  & #2 & {#3}
      & #4   & #5 &\bc{$\pm\ $} &\bc{#6} &\bc{#7}\\}
        \newcommand\insertTableCellWinBold[7]{
        \insertTableCellWin{#1}{#2}{\textbf{#3}}{\textbf{#4}}{\textbf{#5}}{\textbf{#6}}{\textbf{#7}}}
        \newcommand\insertTableCell[7]{ #1 & \multirow{2}{2cm}{#2} & #2 & #3 
      & #4   & #5 &\bc{$\pm\ $} &\bc{#6} &\bc{#7}\\}

        %%%%%%%%%%%%%%%%%%%%%%%%%%%%%%%%%%%%%%%%%%%%%%%%%%%%%%%%%%%%%%%%%%%%%%%
        %%% MAZE 2D
        %%%%%%%%%%%%%%%%%%%%%%%%%%%%%%%%%%%%%%%%%%%%%%%%%%%%%%%%%%%%%%%%%%%%%%%
        \def\sizeCell{2cm}
        \newcommand\hfillPercentage[1]{\hskip 0pt plus #1 fill }

        \def\cssnake{$\SE(3) \times \R^{16}$}
        \def\cssentinel{$\SE(3) \times \R^{80}$}
        \def\chrp{$\SE(2) \times \R^{19}$}
        \def\chrpirr{$\SE(2) \times \R^{5}$}

        \begin{tabularx}{\linewidth}{|X|>{\raggedright\arraybackslash}p{2cm}|>{\raggedright\arraybackslash}p{\sizeCell}|>{\centering\arraybackslash}p{1.3cm}|r@{}l@{}l@{}r@{}l@{}|}
          \hline
          \rowcolor{black!5} \bf Algorithm & \bfseries Manifold & \bfseries Sampling Manifold & \bfseries Success ($\%$) & \multicolumn{5}{c|}{\bfseries Time($s$)}\\
                \insertTableCellHead{Serial Kinematic Chain --- Maze 2D
                Environment (M=100, T=3600)}
                \insertTableCell
                        {\rrt}
                        {$\SE(2)\times\R^3$}
                        {100}
                        {174}{.70}{177}{.07}
                \insertTableCellWinBold
                        {\rrt \irreducible}
                        {$\SE(2)$}
                        {100}
                        {17}{.92}{9}{.47}
        %\end{tabularx}
        %%%%%%%%%%%%%%%%%%%%%%%%%%%%%%%%%%%%%%%%%%%%%%%%%%%%%%%%%%%%%%%%%%%%%%%
        %%% ROCKS 2D
        %%%%%%%%%%%%%%%%%%%%%%%%%%%%%%%%%%%%%%%%%%%%%%%%%%%%%%%%%%%%%%%%%%%%%%%

                \insertTableCellHead{Serial Kinematic Chain --- Rock 2D
                Environment (M=100, T=3600)}
                %\insertTableCell
                %        {\kpiece}
                %        {$\SE(2)\times\R^6$}
                %        {0}
                %        {3600}{.00}{0}{.00}
                %\insertTableCellWin
                %        {\kpiece \irreducible}
                %        {$\SE(2)$}
                %        {100}
                %        {140}{.32}{92}{.69}
                %\hline
                %\insertTableCell
                %        {\pdst}
                %        {$\SE(2)\times\R^6$}
                %        {0}
                %        {3600}{.00}{0}{.00}
                %\insertTableCellWin
                %        {\pdst \irreducible}
                %        {$\SE(2)$}
                %        {100}
                %        {1}{.96}{2}{.15}
                %\hline
                \insertTableCell
                        {\rrt}
                        {$\SE(2)\times\R^6$}
                        {100}
                        {115}{.63}{148}{.75}
                \insertTableCellWinBold
                        {\rrt \irreducible}
                        {$\SE(2)$}
                        {100}
                        {4}{.24}{3}{.17}

        %\end{tabularx}
        %%%%%%%%%%%%%%%%%%%%%%%%%%%%%%%%%%%%%%%%%%%%%%%%%%%%%%%%%%%%%%%%%%%%%%%
        %%% ROCKS 3D
        %%%%%%%%%%%%%%%%%%%%%%%%%%%%%%%%%%%%%%%%%%%%%%%%%%%%%%%%%%%%%%%%%%%%%%%
        %\begin{tabularx}{\linewidth}{X|>{\raggedright\arraybackslash}p{\sizeCell}|>{\centering\arraybackslash}p{\sizeCell}|r@{}l@{}l@{}r@{}l@{}}
                \insertTableCellHead{Serial Kinematic Chain --- Rocks 3D
                Environment (M=$100$, T=$3600$s)}
                %\insertTableCell
                %        {\kpiece}
                %        {\cssnake}
                %        {0}
                %        {3600}{.00}{0}{.0}
                %\insertTableCell
                %        {\kpiece \irreducible}
                %        {$\SE(3)$}
                %        {81}
                %        {819}{.51}{2268}{.28}
                %\hline
                %\insertTableCell
                %        {\pdst}
                %        {\cssnake}
                %        {0}
                %        {3600}{.00}{0}{.0}
                %\insertTableCell
                %        {\pdst \irreducible}
                %        {$\SE(3)$}
                %        {100}
                %        {419}{.55}{483}{.85}
                %\hline
                \insertTableCell
                        {\rrt}
                        {\cssnake}
                        {0}
                        {3600}{.00}{0}{.0}
                \insertTableCellWinBold
                        {\rrt \irreducible}
                        {$\SE(3)$}
                        {100}
                        {531}{.40}{623}{.32}

        %\end{tabularx}
        %%%%%%%%%%%%%%%%%%%%%%%%%%%%%%%%%%%%%%%%%%%%%%%%%%%%%%%%%%%%%%%%%%%%%%%
        %%% SNAKE TURBINE 3D
        %%%%%%%%%%%%%%%%%%%%%%%%%%%%%%%%%%%%%%%%%%%%%%%%%%%%%%%%%%%%%%%%%%%%%%%
                \insertTableCellHead{Mechanical Snake --- Turbine Environment
                (M=$100$, T=$1200$s)}

                \insertTableCell
                        {\kpiece}
                        {\cssnake}
                        {52}
                        { 625}{.74}{566}{.55}
                \insertTableCellWinBold
                        {\kpiece \irreducible}
                        {$\SE(3)$}
                        {100}
                        {13}{.42}{64}{.83}
                \hline
                \insertTableCell
                        {\pdst}
                        {\cssnake}
                        {15}
                        {1120}{.35}{208}{.10}
                \insertTableCellWin
                        {\pdst \irreducible}
                        {$\SE(3)$}
                        {96}
                        {95}{.43}{250}{.40}
                \hline
                \insertTableCell
                        {\rrt}
                        {\cssnake}
                        {46}
                        {813}{.32}{481}{.19}
                \insertTableCellWin
                        {\rrt \irreducible}
                        {$\SE(3)$}
                        {90}
                        {352}{.40}{392}{.35}
                \hline
                \insertTableCell
                        {\sst}
                        {\cssnake}
                        {51}
                        {811}{.40}{462}{.88}
                \insertTableCellWin
                        {\sst \irreducible}
                        {$\SE(3)$}
                        {87}
                        {360}{.60}{436}{.69}
        %%%%%%%%%%%%%%%%%%%%%%%%%%%%%%%%%%%%%%%%%%%%%%%%%%%%%%%%%%%%%%%%%%%%%%%
        %%% PIPES 3D
        %%%%%%%%%%%%%%%%%%%%%%%%%%%%%%%%%%%%%%%%%%%%%%%%%%%%%%%%%%%%%%%%%%%%%%%

                \insertTableCellHead{Mechanical Octopus --- Pipes Environment
                (M=$100$, T=$1200$s)}
                \insertTableCell
                        {\kpiece}
                        {\cssentinel}
                        {0}
                        { 1200}{.00}{0}{.0}
                \insertTableCellWin
                        {\kpiece \irreducible}
                        {$\SE(3)$}
                        {100}
                        {308}{.72}{60}{.78}
                \hline
                \insertTableCell
                        {\pdst}
                        {\cssentinel}
                        {0}
                        {1200}{.00}{0}{.0}
                \insertTableCellWinBold
                        {\pdst \irreducible}
                        {$\SE(3)$}
                        {100}
                        {23}{.88}{14}{.26}
                \hline
                \insertTableCell
                        {\rrt}
                        {\cssentinel}
                        {0}
                        {1200}{.00}{0}{.0}
                \insertTableCellWin
                        {\rrt \irreducible}
                        {$\SE(3)$}
                        {97}
                        {215}{.05}{270}{.08}
                \hline
                \insertTableCell
                        {\sst}
                        {\cssentinel}
                        {0}
                        {1200}{.00}{0}{.0}
                \insertTableCellWin
                        {\sst \irreducible}
                        {$\SE(3)$}
                        {99}
                        {110}{.10}{207}{.82}

        %%%%%%%%%%%%%%%%%%%%%%%%%%%%%%%%%%%%%%%%%%%%%%%%%%%%%%%%%%%%%%%%%%%%%%%
        %%% HRP2 DOOR
        %%%%%%%%%%%%%%%%%%%%%%%%%%%%%%%%%%%%%%%%%%%%%%%%%%%%%%%%%%%%%%%%%%%%%%%
                \insertTableCellHead{Humanoid Robot HRP-2 --- Doors (M=$100$,
                T=$1200$s)}
                \insertTableCell
                        {\kpiece}
                        {\chrp}
                        {0}
                        { 1200}{.00}{0}{.0}
                \insertTableCellWin
                        {\kpiece \irreducible}
                        {\chrpirr}
                        {31}
                        {1019}{.67}{316}{.00}
                \hline
                \insertTableCell
                        {\pdst}
                        {\chrp}
                        {0}
                        {1200}{.00}{0}{.0}
                \insertTableCellWin
                        {\pdst \irreducible}
                        {\chrpirr}
                        {19}
                        {1051}{.05}{336}{.31}
                \hline
                \insertTableCell
                        {\rrt}
                        {\chrp}
                        {40}
                        {834}{.38}{467}{.78}
                \insertTableCellWinBold
                        {\rrt \irreducible}
                        {\chrpirr}
                        {100}
                        {166}{.36}{181}{.99}
                \hline
                \insertTableCell
                        {\sst}
                        {\chrp}
                        {59}
                        {746}{.34}{426}{.11}
                \insertTableCellWin
                        {\sst \irreducible}
                        {\chrpirr}
                        {99}
                        {175}{.92}{199}{.39}
                %\hline
        %%%%%%%%%%%%%%%%%%%%%%%%%%%%%%%%%%%%%%%%%%%%%%%%%%%%%%%%%%%%%%%%%%%%%%%
        %%% HRP2 WALL
        %%%%%%%%%%%%%%%%%%%%%%%%%%%%%%%%%%%%%%%%%%%%%%%%%%%%%%%%%%%%%%%%%%%%%%%
                \insertTableCellHead{Humanoid Robot HRP-2 --- Wall (M=10,
                T=$86400$s=$24$h)}
                \insertTableCell
                        {\rrt}
                        {\chrp}
                        {0}
                        {86400}{.00}{0}{.0}
                \insertTableCellWinBold
                        {\rrt \irreducible}
                        {\chrpirr}
                        {100}
                        {6077}{.40}{2128}{.15}
                \hline
        \end{tabularx}

\end{center}
\end{table}
%\end{adjustbox}

%% file: src/conclusion.tex
\section{Conclusion} 

We described the irreducible path space, a novel concept to reduce the
dimensionality of the configuration space. Our main result is given in Theorem
\ref{thm:complete} stating that a motion planning algorithm
using the space of irreducible paths is complete. While this result remains true
if we apply arbitrary constraints, we have focused here exclusively on
collision-free paths.

We have described how to approximate the space of irreducible paths for a serial
kinematic chain by using the space of curvature constrained paths of the root
link. We developed an algorithm to project sublinks into the swept volume of the
root link. This algorithm works for serial kinematic chains with configuration
space $\SE(3) \times \R^{2N}$ or any subset of that. 

We have proven the correctness of this algorithm for a serial chain disk robot
in 2D having revolute joints and equal length between joints. The proof for 3D
robots with spherical joints and arbitrary lengths is subject of future research.

Using the space of irreducible paths, we conducted experiments for several
mechanical systems including a mechanical snake, a mechanical octopus and a
humanoid robot. We compared four state-of-the-art kinodynamic motion planning
algorithms and we showed that each algorithm performs better using the
irreducible path space.

In future work, we will address the generalization to arbitrary constraints, the
automatic discovery of serial kinematic chains, and we will construct the
irreducible path space for more general chain structures, using the serial
kinematic chain as a fundamental building block. Furthermore, we like to apply
the irreducible path space concept to optimal motion planning algorithms \cite{karaman_2011}.

%% file: src/irreducible-proof-single.tex
\section{Proof that 2D Serial Chain on Curvature Constraint Path is
Irreducible\label{sec:irrlinearchain}}

We will show that if the root link of a serial kinematic chain moves on a
$\kappa_N$-curvature constrained path, then there exists at least one sublink
configuration, such that all sublinks are inside the swept
volume of the root link. We first prove this for $N=1$ case, then generalize our
proof to the $N>1$ case.

We consider a serial kinematic chain in the plane, consisting of a disk-shaped
root link of radius $\delta_0$ plus $N$ disk-shaped sublinks of radius
$\delta_1,\cdots,\delta_N$. The chain has $N$ revolute joints centered at the
center of each disk. The distance between joints is $l_0,\cdots,l_{N-1}$, and
the radius of the disks is such that $\delta_i \leq \delta_0$ for any $i>0$. We
will denote by $\theta_1,\cdots,\theta_N$ the configuration of the sublinks. The
configuration space of the serial chain is then $\SE(2) \times \R^N$ whereby
each joint is restricted by joint limits. An $N=2$ serial kinematic chain is
visualized in Fig. \ref{fig:snake}. 

We will prove that if the root link moves on a curvature constrained path on
$\SE(2)$, then there exists at least one configuration
$\stheta_1,\cdots,\stheta_N$ such that the sublinks are inside the volume swept
by the root link. 

The proof consists of two parts. First, we prove the result for a serial
kinematic chain with $N=1$ sublinks. Second, we generalize this result to $N>1$
sublinks. The proofs use only elementary notions from differential geometry of
curves like the osculating circle. A comprehensive introduction to curve
geometry can be found in \cite{banchoff_2015}.

\subsection{Single Link Chain}

\input{images/snake.tex}

Let us consider an $N=1$ serial kinematic chain with
disk links $L_0,L_1$ in the plane $\R^2$, connected by a rotational joint at the
center of $L_0$, with distance $l_0$ to the center of $L_1$.  The rotational
joint has an allowed rotation of $\theta \in [-\tl,\tl]$, whereby $\tl$ is the
upper limit joint configuration and $-\tl$ is the lower limit joint
configuration. Let us denote by $p_0=(p_{0,0},p_{0,1})\in\R^2$ the position of
$L_0$, and by $p_0'$ its orientation.  Let us define a cone $\Kcone =
\{(x_0,x_1)\in \R^2| \|x_1-p_{0,1}\| \leq (x_0-p_{0,0}) \tan{\tl}\}$ with apex
$p_0$, orientation $p_0'$, and aperture $\tl$. See Fig. \ref{fig:curvature}.
Given $L_0$ at $(p_0,p_0')$ let us define the set $\dP$ of all possible
positions of $L_1$ as a circle intersecting $\Kcone$ and the corresponding disk
segment $\Pp$ as a disk intersecting $\Kcone$.

%\begin{equation}
\begin{align}
  \Pp &= \{x \in \R^2| \|x-p_0\|\leq l_0\} \cap \Kcone\\
  \dP &= \{x \in \R^2| \|x-p_0\|=l_0\} \cap \Kcone
\end{align}
%\end{equation}

\noindent whereby $\Pp$ and $\dP$ are visualized in Fig. \ref{fig:curvature}.

Let us construct a functional space $\Fk$ such that all
functions from $\Fk$ starting at $(p_0,p_0')$ will \emph{necessarily} have to leave
$\Pp$ by crossing $\dP$.

We define the functional space 
%\begin{equation}
\begin{align}
  \Phi_2 &= \{ \phi \in C^2\ |\ \phi: [0,1] \rightarrow \R^2\}\\
  \F_{\Pp} &= \{ \tau \in \Phi_2\ |\ \tau(0)=p_0, \tau'(0)=p_0', \tau(1)\notin \Pp\}
  %\Fk &= \{ \tau \in \Phi_2\ |\ \tau(0)=p_0, \tau'(0)=p_0', \tau(1)\notin \Pp\}
\end{align}
%\end{equation}
\noindent whereby $C^2$ is the space of all continuous two times differentiable functions. 
Let $\Fk \subseteq \F_{\Pp}$ be the subspace of all curvature
constrained functions

%and with maximum curvature
\begin{align}
  \Fk &= \{ \tau \in \F_{\Pp}\ |\ \kappa(\tau(s)) \leq \kappa_0 \}\\
  \kappa_0 &= \frac{2\sin(\tl)}{l_0}
\end{align}

\noindent whereby $\kappa(\tau(s))$ is the curvature at $\tau(s)$. The curvature
$\kappa_0$ has been constructed in the following way: first, we observe that for
any point $\tau(s)$ on $\tau$ the curvature is defined by $\kappa_0 =
\frac{1}{R_0}$ whereby $R_0$ is the radius of the osculating circle at
$\tau(s)$\citep{banchoff_2015}. We consider paths parametrized by arc-length,
such that $\tau'(s)\cdot\tau''(s) = 0$.  The center of the osculating circle has
to lie therefore in the direction of vector $\tau''(s)$. We are searching for
the minimal osculating circle, which ensures that all functions will necessarily
leave $\Pp$ through $\dP$. This minimal osculating circle touches the most
extreme point of $\dP$, which we call $x_M$:

\begin{equation}
        \begin{aligned}
                x_M &= (l_0 \cos(\tl),l_0 \sin(\tl))^T
        \end{aligned}
\end{equation}
See also Fig. \ref{fig:curvature} for clarification. The minimal osculating
circle can be found by solving the equation
\begin{equation}
        \begin{aligned}
                \|x_M - (0,R_0)^T\|^2 = R_0^2
        \end{aligned}
\end{equation}
The solution is given by 

\begin{equation}
        \begin{aligned}
                R_0 &= \dfrac{l_0}{2\sin(\tl)}
        \end{aligned}
\end{equation}

%\subsection{Reducibility theorems of $\Fk$}

We are going to prove some elementary properties of the functional space $\Fk$,
which will show the conditions under which we can project the sublinks.

\input{images/proof-curvature.tex}

%%%%%%%%%%%%%%%%%%%%%%%%%%%%%%%%%%%%%%%%%%%%%%%%%%%%%%%%%%%%%%%%%%%%%%%%%%%%%%%
\begin{theorem}
        For all $\tau \in \Fk$ there exists $s_0\in[0,1]$ such that $\tau(s_0) \in
        \dP$ and $\tau(s) \in \Pp$ for all $s \leq s_0$. \label{thm:curvzero}
\end{theorem}
%%%%%%%%%%%%%%%%%%%%%%%%%%%%%%%%%%%%%%%%%%%%%%%%%%%%%%%%%%%%%%%%%%%%%%%%%%%%%%%
Explanation: any path from the functional space $\Fk$ will
leave the region $\Pp$ by crossing $\dP$. Visualized in Fig. \ref{fig:cone}.

\input{images/cone.tex}
\begin{proof}

        Let us decompose the problem into two parts.  First, let us show that
        all circles with center $(0,R)$ and radius $R\geq R_0$ will intersect
        $\dP$.  Second, let us show that all paths from $\Fk$ will necessarily
        leave $\Pp$ by crossing $\dP$, such that there is no path crossing the
        ball $B_R(0,R)$ for given curvature $\kappa=\frac{1}{R}$.

        \begin{itemize}

                \item 
                  
                  By construction the circle with radius $R_0$ intersects $\dP$
                  at the point specified by angle $\tl = \theta(R) = \asin\left(
                  \dfrac{l_0}{2R} \right)$.  Since $\asin$ is monotone
                  increasing on $[0,1]$, $l_0,R \geq 0$ and $l_0 \leq 2R$, we
                  have that $\theta(R)\geq 0$. We have $\tl \geq \theta(R)$
                  since $\asin\left( \dfrac{l_0}{2R_0} \right) \geq \asin\left(
                  \dfrac{l_0}{2R} \right)$ and therefore we can write
                  $\dfrac{l_0}{2R_0} \geq \dfrac{l_0}{2R}$ since $\asin$ is
                  monotone increasing. It follows that $R\geq R_0$.

                \def\Pcurv{P_0 \setminus (B_R(0,R) \cup B_R(0,-R))} 
        
                \item 
                  
                Let us define the left side of $p_0$ as $LD(p_0) = \{(x_0,x_1)
                \in \R^2|(x_0-p_{0,0})\geq 0, (x_1-p_{0,1})\geq 0\}$ and let us
                construct a polygonal chain as defined by \cite{ahn_2012}. For
                a given $R \geq R_0$ start on the boundary of $\Pp$ at point
                $p_0$ and follow direction $p_0'$ until $\dP$ is reached. At
                $\dP$ move along on $\dP$ until the ball with radius $R$ is
                intersected.  This constitutes a polygonal forward chain
                \citep{ahn_2012}. This chain follows the boundary of $\Pp \cap
                LD(p_0)$. Let us apply Lemma $6$ in \cite{ahn_2012}, stating
                that if a forward chain intersects the circle of radius $R$,
                then the reachable region of all paths in $\Fk$ is given by
                $\Pp \cap LD(p_0) \setminus B_R(0,R)$. See Fig.
                \ref{fig:curvature} for visualization. Applying the pocket
                lemma from \cite{agarwal_2002} it follows that no path can
                escape the region $\Pp \cap LD(p_0) \setminus B_R(0,R)$ except
                through $\dP$ or the lower boundary. The same arguments apply
                for the right side of $p_0$ with $RD(p_0) = \{x \in \R^2|
                x_0-p_{0,0} \geq 0,x_1-p_{0,1} \leq 0\}$ and therefore any
                function in $\Fk$ starting in $p_0$ can escape the region
                $\Pcurv \subset \Pp$ only through the arc segment $\dP$. Since
                $\tau(1)\notin\Pp$, the result follows.
                        
        \end{itemize}

\end{proof}
\def\bd{D_{\delta_0}}
\def\bdd{D_{\delta_1}}

Theorem \ref{thm:curvzero} assures that a particle starting at $(p_0,p'_0)$, following
$\tau \in \Fk$ will always cross the arc segment $\dP$. Now we consider the sweeping of disks
$D_{\delta}(p) = \{x\in\R^2|\|x-p\|\leq \delta\}$ with radius $\delta$ along a
path $\tau \in \Fk$. Let us define $L_0 = \bd(p_0)$, $L_1(\theta) =
\bdd(p_1(\theta))$ with $p_1(\theta) = (l_0 \cos(\theta),l_0 \sin(\theta))$. Let
$\oplus$ denote the Minkowski sum.

\begin{theorem}

        Let $L_0 = \bd(p_0)$. Then there exist a $\theta_1 \in [-\tl,\tl]$ with the property that
        for all $\tau \in \Fk$ there exists $s_0 \in [0,1]$ such that $L_1(\theta_1)
        \subset (\tau(s_0) \oplus L_0)$ if $\delta_1 \leq \delta_0$.
        \label{thm:curvdelta}

\end{theorem}

\begin{proof}

        Applying Theorem \ref{thm:curvzero} a function $\tau \in \Fk$ will
        necessarily intersect $\dP$.  Let $\tau(s_0) \in \dP$ be the
        intersection point. Let us choose $p_1(\theta_1) = \tau(s_0)$ as the
        position of link $L_1$. $\theta_1$ can be computed as $\theta_1 = \acos
        \left( \dfrac{(\tau(s_0)-\tau(0))^Tp_0'}{l_0} \right)$. The volume of
        link $L_1(\theta_1)$ is given by $(\tau(s_0) \oplus L_1(\theta_1))$, and
        is smaller than $(\tau(s_0) \oplus L_0)$ exactly when $\delta_1 \leq
        \delta_0$.

\end{proof}

%% file: images/snake.tex
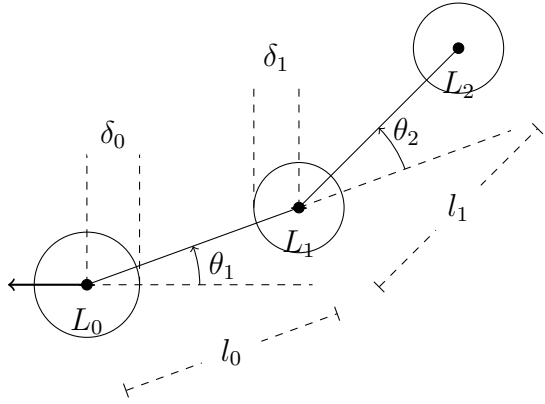
\begin{figure}
        \centering

\begin{tikzpicture}

        \tikzstyle{mcirc}=[draw=black,circle,minimum width=1cm];
        \tikzstyle{bdot}=[draw=black,circle,fill=black];
        \tikzstyle{marc}=[-];

        \def\tI{20}
        \def\tII{45}
        \def\lI{3}
        \def\lII{3}
        \def\dI{0.7cm}
        \def\dII{0.6cm}

        \pgfmathsetmacro\pIx{cos(\tI)*\lI}
        \pgfmathsetmacro\pIy{sin(\tI)*\lI}
        \pgfmathsetmacro\pIIx{cos(\tII)*\lII}
        \pgfmathsetmacro\pIIy{sin(\tII)*\lII}
        \coordinate (p0) at (0,0);
        \coordinate (p1) at ($(p0)+(\pIx,\pIy)$);
        \coordinate (p1p) at ($(p0)+(\lI,0)$);
        \coordinate (p2) at ($(p1)+(\pIIx,\pIIy)$);

        \pgfmathsetmacro\pIIIx{cos(\tI)*\lII}
        \pgfmathsetmacro\pIIIy{sin(\tI)*\lII}
        \coordinate (p2p) at ($(p1)+(\pIIIx,\pIIIy)$);

        \def\zzz{2}

        \coordinate (p0z) at  ($(p0)+(0,\zzz)$);
        \coordinate (p1z) at  ($(p1)+(0,\zzz)$);
        \coordinate (p2z) at  ($(p2)+(0,\zzz)$);

        %%% LINKS

        \draw[mcirc] (p0) circle (\dI);
        \draw[mcirc] (p1) circle (\dII);
        \draw[mcirc] (p2) circle (\dII);

        %%% JOINTS
        \draw[bdot] (p0) circle (2pt) node[below] {$L_0$};
        \draw[bdot] (p1) circle (2pt) node[below] {$L_1$};
        \draw[bdot] (p2) circle (2pt) node[below] {$L_2$};

        \draw[->,thick] (p0) -- ($(p0)+(-1.5*\dI,0)$);
        \draw[marc] (p0) -- (p1);
        \draw[marc] (p1) -- (p2);

        %%% ANGLE DESCRIPTIONS
        \draw[-,dashed] (p0) -- (p1p);
        \draw[-,dashed] (p1) -- (p2p);
        %\draw[->] (p1p) to[bend right] (p1) node
        \def\darcI{0.5*\lI}
        \draw[->] ($(p0)+(0:\darcI)$) arc (0:\tI:\darcI) 
                node[right] at (\tI*0.5:\darcI) {$\theta_1$};
        \def\darcII{0.5*\lII}

        \draw[->] ($(p1)+(\tI:\darcII)$) arc (\tI:\tII:\darcII)
                node[right] at ($(p1)+(\tI+\tII*0.5:\darcII)$) {$\theta_2$};

        %%% DESCRIPTIONS

        \draw[-,dashed] (p0) -- (p0z);
        \draw[-,dashed] (p1) -- (p1z);

        \draw[-,dashed] ($(p1)+(-\dII,0)$) -- ($(p1z)+(-\dII,0)$);
        \draw[-,dashed] ($(p0)+(\dI,0)$) -- ($(p0z)+(\dI,0)$);

        \LeftShiftLine[($(p0)$)][($(p0)+(\dI,0)$)][2cm][$\delta_0$];
        \LeftShiftLine[($(p1)$)][($(p1)+(-\dII,0)$)][-2cm][$\delta_1$];
        \LeftShiftLine[($(p0)$)][($(p1)$)][-1.5cm][$l_0$];
        \LeftShiftLine[($(p1)$)][($(p2)$)][-1.5cm][$l_1$];

\end{tikzpicture}

\caption{$N=2$ serial kinematic chain system}
        \label{fig:snake}
\end{figure}

%% file: images/proof-curvature.tex
\begin{figure}
        \centering
        \def\ll{4}
        \def\ls{0.2cm}
        \def\ttl{30}
        \tikzstyle{mcirc}=[draw=black,circle,minimum width=1cm]
        \tikzstyle{extline}=[-,pattern=dots,dashed,shorten >= -1.7cm]
        \def\dI{0.7cm}
        \def\dII{0.6cm}
        \def\tt{0.05}
        \pgfmathsetmacro\rr{\ll/(2*sin(\ttl))}
        \pgfmathsetmacro\rtheta{asin(\ll/(\rr*2))}

        \def\scaleP{0.65}

        \begin{tikzpicture}[scale=\scaleP]
        \def\ll{3}
        \def\llcone{5}

        \pgfmathsetmacro\llx{cos(\ttl)*\ll}
        \pgfmathsetmacro\lly{sin(\ttl)*\ll}
        \pgfmathsetmacro\lux{cos(\ttl)*\ll}
        \pgfmathsetmacro\luy{-sin(\ttl)*\ll}

        \coordinate (L0) at (0,0);
        \coordinate (L1) at (\ll,0);
        \coordinate (LL) at (\llx,\lly);
        \coordinate (LU) at (\lux,\luy);

        \shade[left color=black!30] ($(L0) + (-\ttl:\llcone)$) arc (-\ttl:\ttl:\llcone)
        -- ($(L0) + (\ttl:0)$) arc (\ttl:0:0) -- cycle;

        \draw[fill=black!10] ($(L0) + (-\ttl:\ll)$) arc (-\ttl:\ttl:\ll)
        -- ($(L0) + (\ttl:0)$) arc (\ttl:0:0) -- cycle;

        \draw[fill=white,ultra thick] ($(L0) + (-\ttl:\ll)$) arc (-\ttl:\ttl:\ll)
        -- ($(L0) + (\ttl:\ll)$) arc (\ttl:-\ttl:\ll) -- cycle;

        \draw[extline] (L0) -- (LL);
        \draw[extline] (L0) -- (LU);

        \draw[mcirc] (L0) circle (\dI);
        \draw[mcirc] (LL) circle (\dII);
        \draw[fill=black] (L0) circle (1pt);
        \draw[fill=black] (LL) circle (1pt);
        \path (L0) -- (L1) node[pos=0.76] {$\Pp$};
        \path (L1) -- (LU) node[pos=0.9,right,xshift=0.05cm] {$\dP$};
        \path (L0) -- (L1) node[pos=1.0,right,xshift=0.05cm] {$\Kcone$};

        %%% S
        \draw[->] (L0) -- ($(L0)+(1,0)$)
        node[pos=0.0,xshift=0.05cm,yshift=0.10cm,below left] {$p_0$};
        %%% S'
        \path (L0) -- ($(L0)+(1,0)$)
        node[pos=1.3,below,xshift=-0.05cm,yshift=0.15cm] {$p_0'$};
        %%% S''
        \draw[->] (L0) -- ($(L0)+(0,1)$) node[pos=1.0,left] {$p_0''$};

        \draw[-,dashed] ($(L0)+(0,-2)$) -- ($(L0)+(0,2)$);
        \end{tikzpicture}
%%%%%%%%%%%%%%%%%%%%%%%%%%%%%%%%%%%%%%%%%%%%%%%%%%%%%%%%%%%%%
%%%%%%%%%%%%%%%%%%%%%%%%%%%%%%%%%%%%%%%%%%%%%%%%%%%%%%%%%%%%%
        \begin{tikzpicture}[scale=\scaleP]
        %%%%%
        \pgfmathsetmacro\llx{cos(\ttl)*\ll}
        \pgfmathsetmacro\lly{sin(\ttl)*\ll}
        \pgfmathsetmacro\lux{cos(\ttl)*\ll}
        \pgfmathsetmacro\luy{-sin(\ttl)*\ll}

        \coordinate (BR) at (0,\ll);
        \coordinate (L1) at (\ll,0);
        \coordinate (L0) at (0,0);
        \coordinate (LL) at (\llx,\lly);
        \coordinate (LU) at (\lux,\luy);

        %%%%%

        \draw[fill=black!10] ($(L0) + (-\ttl:\ll)$) arc (-\ttl:\ttl:\ll)
        -- ($(L0) + (\ttl:0)$) arc (\ttl:0:0) -- cycle;

        \draw[fill=white,ultra thick] ($(L0) + (-\ttl:\ll)$) arc (-\ttl:\ttl:\ll)
        -- ($(L0) + (\ttl:\ll)$) arc (\ttl:-\ttl:\ll) -- cycle;

        \draw[fill=black] (BR) circle (1pt) node[above]{$(0,R_0)$};

        \draw[fill] (L0) circle (1pt);
        \draw[fill=black] (L1) circle (1pt);

        \draw[->] (L0) -- (\ll+1,0) node[below] {$x_0$};
        \draw[->] (L0) -- (0,2) node[left] {$x_1$};
        \draw[-,dashed] (L0) -- (BR);
        
        \draw[-,dashed] (BR) -- (LL) node[pos=0.5,above] {$R_0$};
        \draw[fill=black] (LL) circle (1pt) node[right] {$x_M$};

        \path (L0) -- ($(L0)+(1,0)$) node[pos=0.0,below] {$p_0$};

        \draw[fill=white,thick] ($(L0)+(0,\rr)+(-90:\rr)$) arc (-90:-0:\rr)
        -- ($(L0)+(0,\rr)+(-0:\rr)$) arc (-0:-90:\rr) -- cycle;

        \draw[pattern=dots, pattern color=gray,thick] ($(L0)+(0,\rr)+(-90:\rr)$) arc (-90:-\rtheta:\rr)
        -- ($(L0) + (\ttl:\ll)$) arc (\ttl:0:\ll) -- cycle;

        \draw[->] ($(LU)+(-0.5,-0.2)$)--($(L1)+(-1,0.5)$) node[pos=0.0,below]{$\Pp\cap LD(s)
        \setminus B_R(0,R)$};

        \def\offset{1.0}
        \draw[thick,->] ($(L0)+(0:\ll+\offset)$) arc (0:+\rtheta:\ll+\offset) node[right] at (\rtheta*0.5:\ll+\offset) {$\tl$};

        \end{tikzpicture}

\caption{$\dP$ is the space of all possible positions of link $L_1$, constrained
by link $L_0$. We establish in this section that for a specifically constructed
functional space $\Fk$ any function which starts at $p_0$ and has first derivative
equal to $p_0'$ will leave the area $P$ by crossing $\dP$.\label{fig:curvature}}

\end{figure}
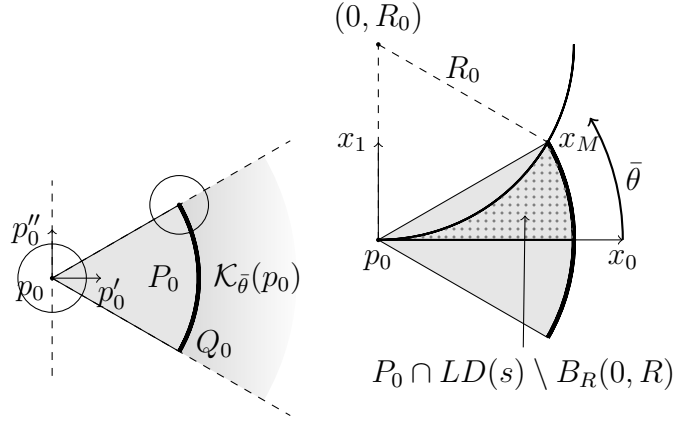

%% file: images/cone.tex
\begin{figure}
        \centering
%\begin{tikzpicture}
%%%arc (startAngle:endAngle:radius and 2ndaxis Radius)
%        \def\Cradius{2}
%        \draw (0,0) arc (-90:90:0.5cm and 2cm) -- (-3,\Cradius) -- cycle;
%        \shade[left color=gray!20!white,right color=gray!40!white,opacity=0.3] 
%              (0,0) arc (-90:90:0.5cm and 2cm) -- (-3,\Cradius) -- cycle;
%        \draw[dashed] (0,0) arc (270:90:0.5cm and 2cm);
%        \draw[fill=black] (0,2*\Cradius) circle (2pt) node[right] {$x_M$};
%        \draw[fill=black] (0,2*\Cradius) circle (2pt) node[right] {$x_M$};
%
%        \path (-3,\Cradius) -- (0,2*\Cradius) node[pos=0.5,above] {$l_0$};
%        \path (-3,\Cradius) -- (0,\Cradius) node[pos=0.5] {$C_s$};
%        \draw[fill=black] (-3,\Cradius) circle (2pt) node[left] {$s$};
%        \draw[->] (-3,\Cradius) -- (-2,\Cradius) node[above,pos=0.8] {$s'$};
%        \draw[->] (-3,\Cradius) -- (-3,\Cradius+1) node[left,pos=0.8] {$s''$};
%\end{tikzpicture}
\includegraphics[width=0.5\linewidth]{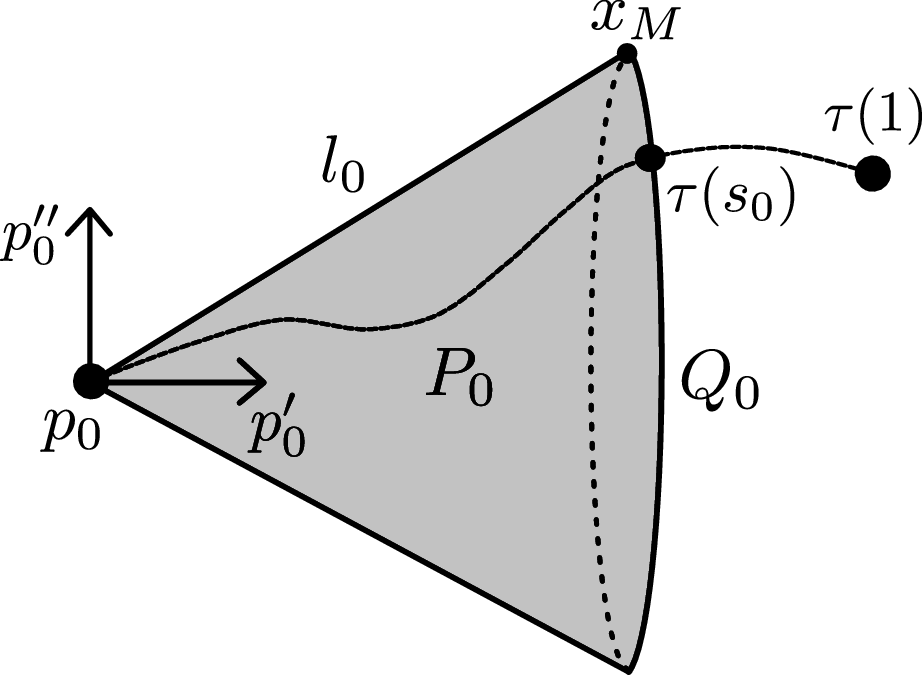}

\caption{Cone spanned by the length $l_0$, the limit angle $\tl$ and the
position of $s$. Every function from $\Fk$ will necessarily leave $\Pp$ by
crossing $\dP$ at $\tau(t_0)$ to reach a point $\tau(1)$ outside $\Pp$.\label{fig:cone}}
\end{figure}

%% file: src/irreducible-proof-multi.tex
\subsection{Multi Link Chain\label{sec:irrproof}}

Let $L_0,\cdots,L_N \in D^2$ be disk links of radius $\delta_0,\cdots,\delta_N$
connected by lines of equal length $l_0,\cdots,l_{N-1}$ with
$l_0=\cdots=l_{N-1}$, $\delta_i > 0$, $l_i > \delta_i+\delta_{i+1}$ (no
overlapping disks), $\delta_i \leq \delta_0$ for all $i>0$ and joint limits
$\{\{-\tl_0,\tl_0\},\cdots,\{-\tl_{N-1},\tl_{N-1}\}\}$ with
$\tl_0=\cdots=\tl_{N-1}$. We will refer to this serial kinematic chain structure
as $\RobotLL$.

Let $V(\theta_0,\cdots,\theta_{N-1})$ be the swept volume of the chain without
the links for a given set of configurations. We define $\Pn$ as the union of all
swept volumes of the chain under the constraint that $-\tl \leq \theta_i \leq
\tl$ for every $i \in [0,N-1]$.  This is depicted in Fig.
\ref{fig:conesuccession}. Further, let $\dPn$ be the part of the outer border
which we obtain by removing the swept volume of the minimum joint configuration
$V(\theta_0=-\tl,\cdots,\theta_{N-1}=-\tl)$, and the swept volume of the maximum
joint configuration $V(\theta_0=-\tl,\cdots,\theta_{N-1}=-\tl)$, from the
boundary of $\Pn$.  $\dPn$ is shown in Fig. \ref{fig:conesuccession}.

As in the $N=1$ case, let us construct a functional space $\Fkn$ as

\begin{equation}
        \begin{aligned}
          \F_{\Pn} &= \{ \tau \in \Phi_2\ |\ \tau(0)=p_0, \tau'(0)=p_0', \tau(1)\notin \Pn\}
        \end{aligned}
\end{equation}

\noindent Let $\Fkn \subseteq \F_{\Pn}$ be the subspace of curvature constrained functions
%and for all $\tau \in
%\Fkn$ we have a maximum curvature given by 

\begin{align}
    \Fkn &= \{ \tau \in \F_{\Pn}\ |\ \kappa(\tau(s)) \leq \kappa_N \}\\
    \kappa_N &= \frac{2\sin(\tl)}{\sumL}, N>1
\end{align}

\begin{figure}
        \centering
        \includegraphics[width=0.6\linewidth]{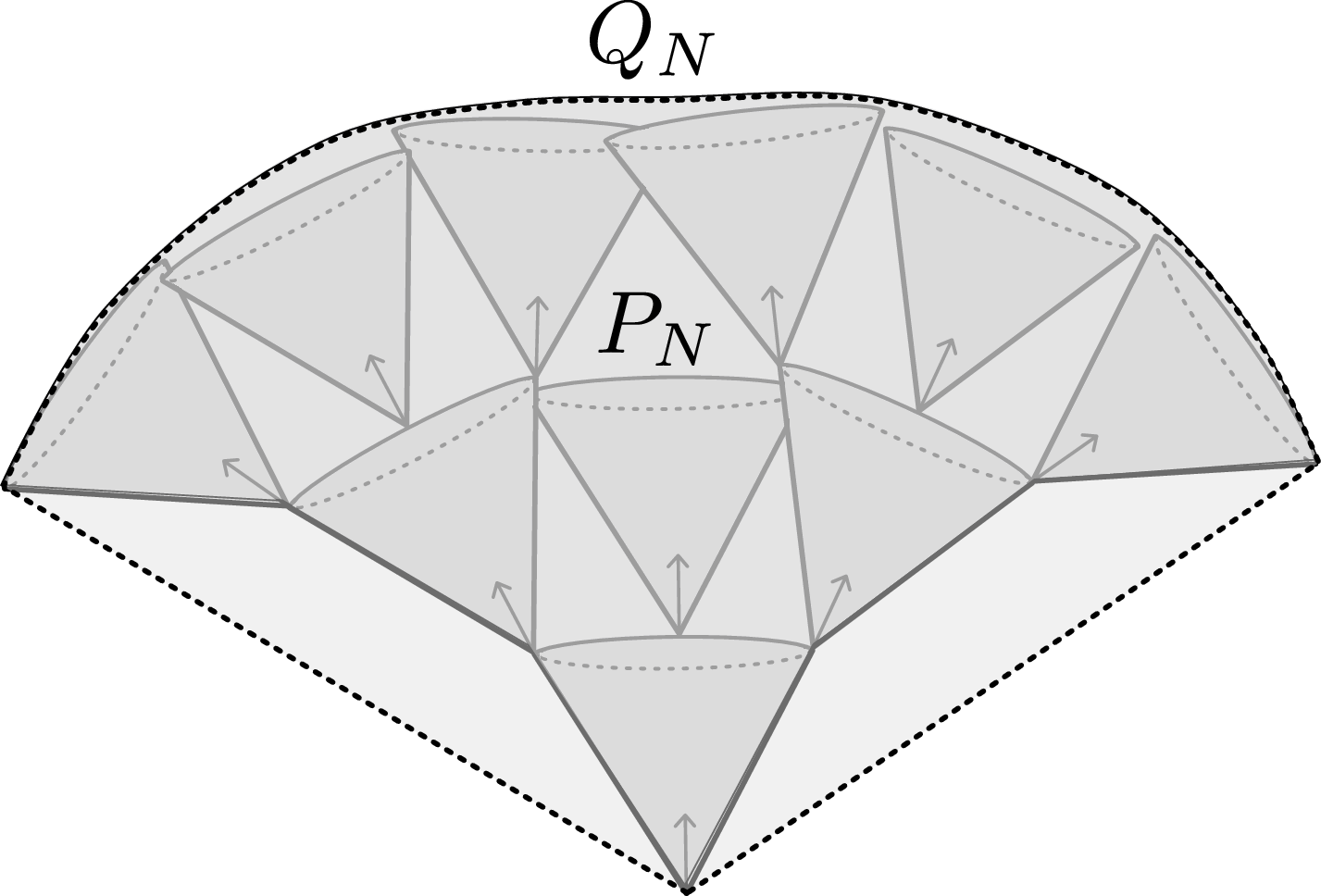}
        \caption{A succession of cones, spanning the space between $s$ and
                $\dPn$, which necessarily has to be traversed by any
        function from $\Fkn$.\label{fig:conesuccession}}
\end{figure}

%\subsection{Irreducibility of serial kinematic chain\label{sec:irrproof}}

For $N=1$, we proved that there exist $\theta_0$ such that $L_1(\theta_0) \in \tau \oplus L_0(p_0)$. For
$N>1$, we need to take into account the change of orientation when the point
has moved from $L_0$ to $L_1$. At $L_1$, we need to make sure that the obtained
orientation $\theta_0$ and the next orientation $\theta_1$ are both below the
maximum orientation $\tl$.
See Fig.  \ref{fig:nll} for clarification. 

To ensure that we can always find a feasible
configuration, such that all links are on $\tau$, we therefore need to ensure
that $\theta_i + \theta_{i-1} \leq \tl$ for all $i \in[1,N]$.

\begin{figure}[h!]
\centering
\includegraphics[width=0.6\linewidth]{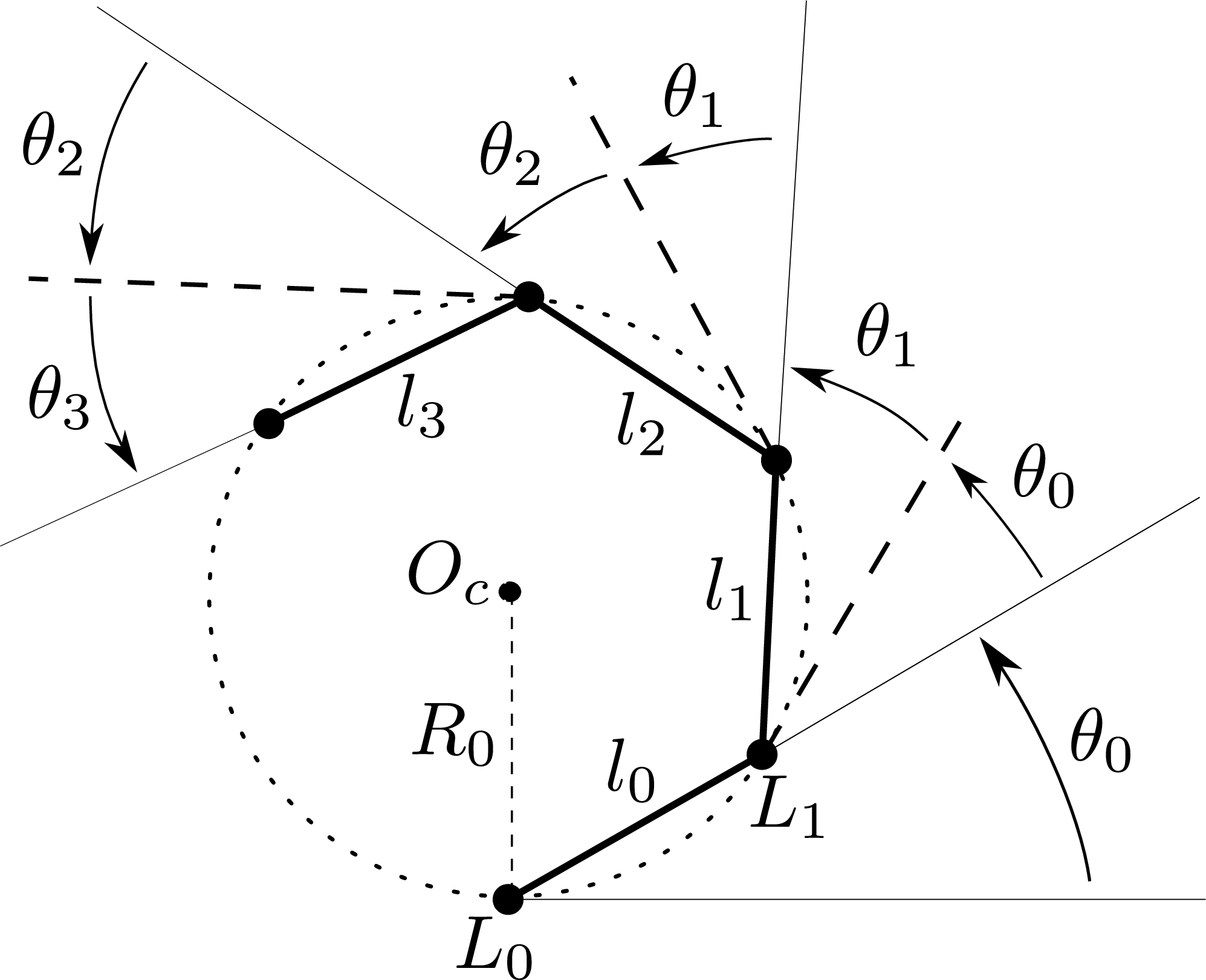}
\caption{\label{fig:nll}Serial kinematic chain along a curve $\tau$ with
constant curvature $\kappa_N$.}
\end{figure}

%
%\begin{enumerate}[{\bf P1}]
%        \item If $l_0=l_i$ for all $i \in [1,N]$, then $\td=\tdi$ for all $i \in [1,N]$
%        \item Maximum angle between $t$ and $n$ can be found for $\tau =
%                \tau_{\kappa_N}$ with $\tau_{\kappa_N}$ being the constant
%                maximum curvature path with curvature $\kappa_N$
%                everywhere.
%\end{enumerate}

We like to rewrite all angles in terms of the radius of the osculating circle
$R_0$ and the length of the links $l_i$. 
The angle $\ti$ can be readily expressed as 

\begin{equation}
  \small
  \begin{aligned}
    \ti(R_0,l_i) &= \arctan{\left(\dfrac{l_i}{\sqrt{4R_0^2-l_i^2}}\right)}\\
  \end{aligned}
\end{equation}

This expression is obtained by centering a coordinate system at the center
$O_c=(0,0)$ of the circle, then drawing two circles, one about $O_c$ with radius
$R_0$, one about $L_1$ with radius $l_0$. Since we have chosen $R_0$ such that
$R_0 > l_0$, those circles have two intersection points. The two intersection
points together with $O_c$ and $L_1$ create a geometric kite, which can be
analyzed by geometrical inspection to arrive at the equation, see circle-circle
intersection
\footnote{\href{http://mathworld.wolfram.com/Circle-CircleIntersection.html}{Circle-Circle
Intersection -- Wolfram Mathworld}}.

%By geometrical arguments of circle-circle intersection
%\footnote{\href{http://mathworld.wolfram.com/Circle-CircleIntersection.html}{Circle-Circle
%Intersection -- Wolfram Mathworld}}, we can write
%\def\Asqrt{\sqrt{4R_0^2-l_0^2}}
%\begin{equation}
%  \small
%\begin{aligned}
%        n(R_0,l_0) &= \begin{pmatrix} \frac{l_0}{2R_0}\Asqrt\\
%                                        \frac{l_0^2}{2R_0}
%                        \end{pmatrix}\\
%        t(R_0,l_0) &= \begin{pmatrix} \frac{-l_0^2}{2R_0}+R_0\\
%                                        \frac{l_0}{2R_0}\Asqrt
%                        \end{pmatrix}\\
%\end{aligned}
%\end{equation}

%\begin{equation}
%  \small
%\begin{aligned}
%        \tt &= \arctan{\left(\dfrac{l_0}{\sqrt{4R_0^2-l_0^2}}\right)}\\
%        \td &= \arccos{\left(\dfrac{n\cdot t}{\|n\|\|t\|}\right)} =
%        \arccos{\left(\dfrac{4R_0^2-l_0^2}{4R_0^2}\right)}
%\end{aligned}
%\end{equation}

Let
\begin{equation}
\begin{aligned}
  %R_0 &= \dfrac{\sum\limits_{i=0}^{N-1} l_i}{2\sin{\tl}}
  R_0 &= \dfrac{N l_0}{2\sin{\tl}}
\end{aligned}
\end{equation}

such that $\Fkn$ is defined by $\kappa_N = \frac{1}{R_0}$.

\begin{lemma}

%Given premises \preone, \pretwo and 

  Given a path $\tau \in \Fkn$ with $\frac{1}{\kappa_N} = R_0 = \dfrac{N
  l_0}{2\sin{\tl}}$ and $N>1$, there exist joint configurations
  $\theta_1,\cdots,\theta_N$ for the serial kinematic chain $\RobotLL$, such
  that every $L_i$ is located on $\tau$.  Furthermore, the maximum distance
  between $\tau$ and the lines $(L_0L_1)\cdots (L_{N-1}L_N)$ is given by
  \begin{equation}
    \begin{aligned}
      \dkn=R_0 - \sqrt{R_0^2-\dfrac{l_0^2}{4}}
    \end{aligned}
  \end{equation}
  
  %given by $\dkn=\underset{i \in {0,N-1}}{\max}\left(R_0 -
  %\sqrt{R_0^2+\dfrac{l_i^2}{4}}\right)$

\label{thm:conftau}

\end{lemma}

\begin{proof}

  Evaluating $\ti$ at $R_0$ gives
\begin{equation}
  \small
\begin{aligned}
  \ti(R_0,l_0) = \ti(N) = \arctan{\left(\dfrac{\sstl}{\sqrt{N^2 - \stl}}\right)}
        %\td &= \arccos{\left(\dfrac{N^2-\stl}{N^2}\right)}
    %\ti(R_0,l_i) &= \arctan{\left(\dfrac{l_i}{\sqrt{4R_0^2-l_i^2}}\right)}\\
  %R_0 &= \dfrac{\sum\limits_{i=0}^{N-1} l_i}{2\sin{\tl}}
\end{aligned}
\end{equation}
for $N>1$. By induction on $N$, we get for $N=2$

%Due to premise \pretwo, we know that $\tt+\td(R_0) \geq \tt+\td(R)$ for
%$R\geq R_0$, and so we can concentrate on the maximum curvature case $R_0$. Due
%to premise \preone, we now only have to prove that $\tt+\td \leq \tl$.

\begin{equation}
  \small
\begin{aligned}
\ti(2) &= 
        \arctan{\left(\dfrac{\sstl}{\sqrt{4 - \stl}}\right)}\leq
\arctan{\left(\dfrac{\sstl}{2}\right)}\\&\leq
\dfrac{\sstl}{2} \leq
\dfrac{\tl}{2} \\
%\td(2) &= 
%%\arccos{\left(1-\dfrac{\stl}{4}\right)}=
%\arccos{\left(1-\dfrac{\stl}{4}\right)}=
%2\arctan{\left(\dfrac{2\sstl}{8-\stl}\right)}\\&\leq
%\dfrac{4\sstl}{8-\stl}\leq
%\dfrac{4\sstl}{8}=
%\dfrac{\sstl}{2}\leq
%\dfrac{\tl}{2}
\end{aligned}
\end{equation}
whereby we relied on the fact that for $x>0$ we have $\arctan(x) \leq x$ since $\arctan'(x) =
\frac{1}{1+x^2} \leq 1$, for $x>0$ we have $\sin(x) \leq x$ since
$\sin'(x)=\cos(x) \leq 1$.
%and that
%$\arccos(x)=2\arctan{\left(\dfrac{\sqrt{1-x^2}}{1+x}\right)}$.

We now observe that
\begin{equation}
  \small
\begin{aligned}
\ti(N) &= 
        \arctan{\left(\dfrac{\sstl}{\sqrt{N^2 - \stl}}\right)}\geq 
\arctan{\left(\dfrac{\sstl}{N}\right)}\\&\geq
        \arctan{\left(\dfrac{\sstl}{\sqrt{(N+1)^2-\stl}} \right)} = \ti(N+1)\\
%\td(N) &= 
%\arccos{\left(\dfrac{N^2-\stl}{N^2}\right)}\geq 
%\arccos{\left(1-\dfrac{\stl}{N^2}\right)}\\&\geq 
%\arccos{\left(1-\dfrac{\stl}{(N+1)^2}\right)}=
%\td(N+1)
\end{aligned}
\end{equation}
which shows that $\ti(N) + \tii(N) \geq \ti(N+1) + \tii(N+1)$.
Therefore $\tl \geq \ti(2) + \tii(2) \geq \cdots \geq \ti(N) + \tii(N)$
for $N>1$ as required.

%which shows that $\tt(N) + \td(N) \geq \tt(N+1) + \td(N+1)$.
%Therefore we have $\tl \geq \tt(2) + \td(2) \geq \cdots \geq \tt(N) + \td(N)$
%for $N>1$ as required.

Given the constant maximum curvature $\kappa_N$, the points $L_i,L_{i-1}$ and
$O_c$ are creating an isosceles triangle. See Fig.  \ref{fig:nll} for
clarification. The maximum distance $d_i$ of the line $(L_iL_{i-1})$ and the
circle can therefore be obtained by subtracting the height of the isosceles
triangle from the radius of the circle as $\dkn = R_0 - \sqrt{R_0^2 -
\dfrac{l_i^2}{4}}$. 

%Taking the maximum over all links we obtain
%$\dkn=\underset{i \in {0,N-1}}{\max}\left(R_0 -
%\sqrt{R_0^2+\dfrac{l_i^2}{4}}\right)$.

\end{proof}

The swept volume of the root link $L_0$ will be the Minkowski sum with the
path $\tau$, i.e. $\SV_{L_0}(\tau) = L_0 \oplus \tau$.
However, the sublinks will not lie inside of this swept volume at the starting
configuration. To circumvent this problem we imagine the path $\tau$ being extended
along the positions of the sublinks at the start configuration.

We call this extended part $\tau_I \in \Fkn$. $\tau_I$ can be obtained
by computing a path which starts at $s=0$ at the position of the sublink
$L_N$ at the start configuration, and follows each sublink $L_i$ until it
reaches $L_0$ at instance $s=1$, such that $\tau_I(1)=\tau(0)$ and
$\dtau_I(1)=\dtau(0)$. 

We claim that along $\tau_I \circ \tau_{\kappa_N}$ we can find at least one configuration
of the sublinks, such that the volume of the sublinks is inside the swept volume
of the root link.

\begin{theorem}

Let $\tau = \tau_I \circ \tau_{\kappa_N} \in \Fkn$. If the root link $L_0$ moves
along $\tau_{\kappa_N}$, then for
$\delta_i \leq \delta_0$ and $\dkn \leq \delta_0$, we have that there exists at
least one configuration $\stheta_1(s),\cdots,\stheta_N(s)$ for any $s \in [0,1]$ such
that the volume of the serial kinematic chain $\RobotLL$ is a subset of $\tau \oplus
L_0$ \label{thm:llreduce}

\end{theorem}

\begin{proof}

By Theorem \ref{thm:conftau} $\stheta_1(s),\cdots,\stheta_N(s)$ can be chosen such
that the center of every $L_i(\stheta_1(s),\cdots,\stheta_i(s))$ is located on
$\tau$. Then there exists an instance $s_i \in [0,1]$ such that
$L_i(\stheta_1(s),\cdots,\stheta_i(s)) = \tau(s_i)$.
$L_i(\stheta_1(s),\cdots,\stheta_i(s))$ is a subset of $\tau \oplus L_0$ if
$\delta_0 \geq \delta_i$.  By Lemma \ref{thm:conftau}, the maximum distance of
the serial kinematic chain at $\stheta_1(s),\cdots,\stheta_N(s)$ to $\tau$ is
given by $\dkn$. If $\delta_0 \geq \dkn$, then any point on the serial kinematic
chain curve will be inside $\tau \oplus L_0$.  
\end{proof}